\pdfoutput=1
\documentclass[10pt,twocolumn,letterpaper]{article}
\usepackage[pagenumbers]{cvpr} %
\usepackage{algorithm}
\usepackage[noend]{algpseudocode}
\DeclareMathOperator*{\argmin}{arg\,min}

\usepackage{amsthm}
\newtheorem{proposition}{Proposition}
\newcommand{\tbest}[1]{\textbf{#1}}
\newcommand{\tsecond}[1]{\underline{#1}}
\usepackage{multirow}
\usepackage{makecell}
\usepackage{colortbl}
\usepackage{stfloats}
\definecolor{oursw}{RGB}{252,229,239}
\definecolor{oursv}{RGB}{238,232,250}
\definecolor{cvprblue}{rgb}{0.21,0.49,0.74}
\usepackage[breaklinks,colorlinks,allcolors=cvprblue]{hyperref}
\def\paperID{*****}
\def\confName{CVPR}
\def\confYear{2027}

\title{SpikeQ: Spike-Guided Post-Training Quantization for Spiking Neural Networks}
\author{
Kai Sun\textsuperscript{1}, Peibo Duan\textsuperscript{1,*}, Zhanglu Yan\textsuperscript{3}, Guowei Zhang\textsuperscript{1}, Yongsheng Huang\textsuperscript{2}, Levin Kuhlmann\textsuperscript{1,*}\\
{\normalsize \textsuperscript{1}Monash University \quad \textsuperscript{2}Northeastern University \quad \textsuperscript{3}National University of Singapore}\\
{\tt\small \{kai.sun1, peibo.duan, levin.kuhlmann\}@monash.edu}
}

\begin{document}
\maketitle

\begin{abstract}
On neuromorphic chips, spiking neural networks (SNNs) consume far less energy than artificial neural networks (ANNs), but the limited memory of these chips requires low-bit synaptic weights and membrane potentials. Existing post-training quantization (PTQ) methods provide them without retraining and, like those for ANNs, minimize numerical errors such as the mean squared error (MSE). In an SNN, however, a quantization error reaches later layers only through the spikes it flips. We find that errors of equal MSE can cause very different accuracy losses, whereas errors that flip the same fraction of spikes cause the same loss. We therefore propose \textbf{SpikeQ}, a spike-guided PTQ framework. For weights, we select the level of each weight by simulating the downstream neurons and minimizing their spike errors. For membrane potentials, most existing methods quantize the state after the reset, whereas neuromorphic chips compare the threshold with the membrane potential stored at a fixed bit width. We target this pre-fire quantization and build a reset-aware grid that places the threshold on a rounding boundary, so rounding never flips the spike at the current time step. We evaluate SpikeQ on seven benchmarks of image classification, semantic segmentation, and object detection. With 2-bit pre-fire membrane potentials, every existing grid collapses, whereas SpikeQ retains most of the full-precision accuracy. With both at 4~bits, the models stay within 2.59 points of full precision, shrink their memory footprint \mbox{7.8--7.9$\times$}, and use \mbox{6.4--6.9$\times$} less energy.
\end{abstract}

\section{Introduction}
\begin{figure}[t]
\centering
\includegraphics[width=\columnwidth]{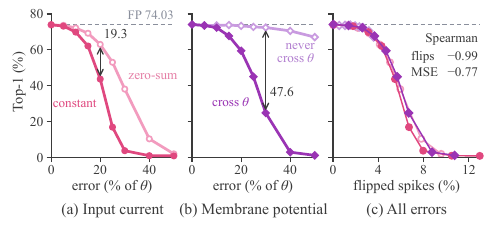}
\caption{Accuracy of a full-precision SNN (PreAct-ResNet18, CIFAR-100) under injected errors of the same MSE. (a) Constant versus zero-sum errors in the input current. (b) Membrane potential errors that cross $\theta$ versus never cross $\theta$. (c) Accuracy versus the fraction of flipped spikes for all errors in (a) and (b).}
\label{fig:motivation}
\end{figure}

Spiking neural networks (SNNs) have attracted growing interest as an energy-efficient alternative to artificial neural networks (ANNs). Instead of dense multiply-accumulate operations, SNNs compute with sparse spikes. On neuromorphic chips such as Loihi and Speck~\cite{davies2018loihi,yao2024speck}, this event-driven computation consumes substantially less energy than ANN inference on conventional hardware~\cite{roy2019towards,horowitz2014computing}. Recent work has further scaled SNNs to deep residual networks and spiking transformers~\cite{zheng2021going,fang2021deep,yao2023spike}, narrowing their accuracy gap to ANNs. As these models grow, however, their floating-point synaptic weights and membrane potentials become a major deployment cost. Neuromorphic chips store both quantities in limited on-chip memory and repeatedly read and update the membrane potential of each neuron. Quantizing both weights and membrane potentials is therefore necessary to run large SNNs on such chips without losing their energy advantage~\cite{shen2024conventional}.

SNN quantization methods fall into two paradigms, namely quantization-aware training (QAT) and post-training quantization (PTQ). QAT simulates quantization during training and learns low-precision weights and membrane potentials with surrogate gradients~\cite{wei2024qsnn,yin2024mint,hasssan2024spquant}, but it repeats the full training through time for every bit width and requires the original training data. PTQ instead calibrates a pretrained SNN on a small unlabeled set without retraining, which makes it more practical for large SNNs. Current PTQ methods for SNNs largely follow those developed for ANNs. Weight quantization either learns the rounding by reconstructing the output of a layer or block~\cite{nagel2020adaround,li2021brecq,xie2026ptq4snn} or compensates the rounding error with second-order information~\cite{frantar2023gptq,mei2026soptq,shi2026sbc}. Membrane quantization treats the membrane potential as an activation and uses a uniform grid whose scale is fitted to its distribution or tied to the weight scale~\cite{flowq,xie2026ptq4snn}. In both cases, quantization is calibrated by minimizing numerical errors.

Unlike ANNs, SNNs transmit information between neurons only through spikes, and each neuron fires only when its membrane potential, accumulated over time steps, crosses the threshold. A quantization error in the input current or membrane potential therefore reaches later layers through the spikes it changes, now or at later time steps, yet existing PTQ objectives measure only its magnitude, such as the mean squared error (MSE). Errors of the same MSE can change spikes very differently. For example, a constant input-current error of $+0.2\theta$ at every time step keeps raising the membrane potential and can push it across the threshold, whereas a zero-sum one alternating between $\pm0.2\theta$ has the same MSE but largely cancels out over time. To test whether MSE or spikes determine accuracy, we inject errors of the same MSE into every neuron of a full-precision SNN (\cref{fig:motivation}). For the input current, which weight quantization perturbs, constant errors lower accuracy far more than zero-sum ones (\cref{fig:motivation}(a)). For the membrane potential, we reverse the sign of an error whenever it would push the membrane potential across the threshold. This keeps the MSE unchanged but reduces the sharp accuracy drop of the original errors to almost none (\cref{fig:motivation}(b)). Plotted against the fraction of flipped spikes, i.e., spikes added or removed by the error, the four curves of \cref{fig:motivation}(a) and (b) merge into one (\cref{fig:motivation}(c)), so the same fraction of flipped spikes causes the same accuracy loss. The Spearman correlation between the two is $-0.99$, far stronger than that of MSE, so more flipped spikes almost always means lower accuracy. Spikes are therefore a better guide than MSE for SNN quantization.

In this paper, we propose SpikeQ, a spike-guided PTQ framework that uses spikes rather than numerical errors to guide the quantization of both weights and membrane potentials without gradient-based fine-tuning. Our contributions are summarized below.
\begin{itemize}
\item We show that the accuracy loss caused by quantization errors follows the fraction of spikes they flip rather than their MSE, which existing PTQ methods minimize.
\item We propose SpikeQ, which quantizes both weights and membrane potentials by minimizing spike errors relative to full precision and makes the threshold a rounding boundary, so rounding never flips the current spike.
\item Extensive experiments on image classification, semantic segmentation, and object detection show that SpikeQ outperforms existing PTQ methods, and its 4-bit models lose at most 2.59 points while shrinking the memory footprint \mbox{7.8--7.9$\times$} and using \mbox{6.4--6.9$\times$} less energy.
\end{itemize}

\section{Related Work}
\paragraph{SNN Quantization.}
SNN quantization follows two paradigms, QAT and PTQ. QAT simulates quantization in the forward pass, as in ANNs~\cite{jacob2018quantization,esser2020lsq}, and trains the SNN by backpropagation through time with surrogate gradients~\cite{wu2018spatio,neftci2019surrogate}. It has produced low-bit, binary, and even sub-bit weights~\cite{eshraghian2022navigating,shymyrbay2023low,wei2025qpsnn,kheradpisheh2022bs4nn,liang2025agmm,cao2025bestformer,wei2025s2nn} as well as low-precision membrane potentials~\cite{yin2024mint,liu2026hardfsnn,wei2024qsnn,hasssan2024spquant,venkatesh2024squat}, with the LIF neuron and distillation adapted to quantized SNNs~\cite{qiu2025qsd,lee2026mdsnn}. However, this training has to be repeated on the original data for every bit width, and each run backpropagates through all time steps, which is costly for large SNNs. PTQ instead quantizes a pretrained SNN using a small unlabeled calibration set without retraining. Existing PTQ methods allocate bit widths across layers or neurons~\cite{li2025neuronquant,putra2025qsvit} or calibrate the quantization of weights alone~\cite{mei2026soptq,shi2026sbc} or with membrane potentials~\cite{flowq,xie2026ptq4snn}. In this paper, we focus on PTQ.

\paragraph{Weight Quantization.}
Weight PTQ for SNNs selects the rounding by aligning one of two quantities between the full-precision and quantized networks, the input current or the membrane potential (\cref{fig:problem}(a)). Current alignment minimizes the error of each layer output, which is the input current of its neurons. AdaRound and GPTQ have been applied directly to SNNs for this purpose~\cite{nagel2020adaround,frantar2023gptq,xie2026ptq4snn}, and S-OPTQ adapts GPTQ to spiking vision transformers~\cite{mei2026soptq}. Membrane alignment instead minimizes the error of the membrane potential before firing. SBC defines its objective on output spike trains, but with a constant surrogate gradient its Hessian equals that of this error for a neuron that never fires or resets~\cite{shi2026sbc}. However, their rounding ignores how the errors change spikes, which pass to the next time step and layer. A large error far from the threshold changes no spike, whereas a small error near the threshold flips one.

\paragraph{Membrane Quantization.}
Most SNN quantization methods quantize the membrane potential after the reset, that is, the state carried to the next time step. Their membrane grids follow activation quantization and differ mainly in how the scale, the spacing between quantization levels, is chosen. Some fit the scale to the range of the membrane potential~\cite{wei2024qsnn,hasssan2024spquant}, and others derive it from the scale of the weights~\cite{yin2024mint,liu2026hardfsnn,flowq,xie2026ptq4snn}. However, neuromorphic chips store the membrane potential at a fixed bit width and compare this stored value with the threshold~\cite{merolla2014million,davies2018loihi,yao2024speck}, so on chip the membrane potential is quantized before firing. Applied there, these grids leave the threshold inside a rounding interval, even when they concentrate levels near it~\cite{venkatesh2024squat}, so rounding can push a near-threshold membrane potential across the threshold and flip the spike.

\begin{figure*}[t]
\centering
\includegraphics[width=0.95\textwidth]{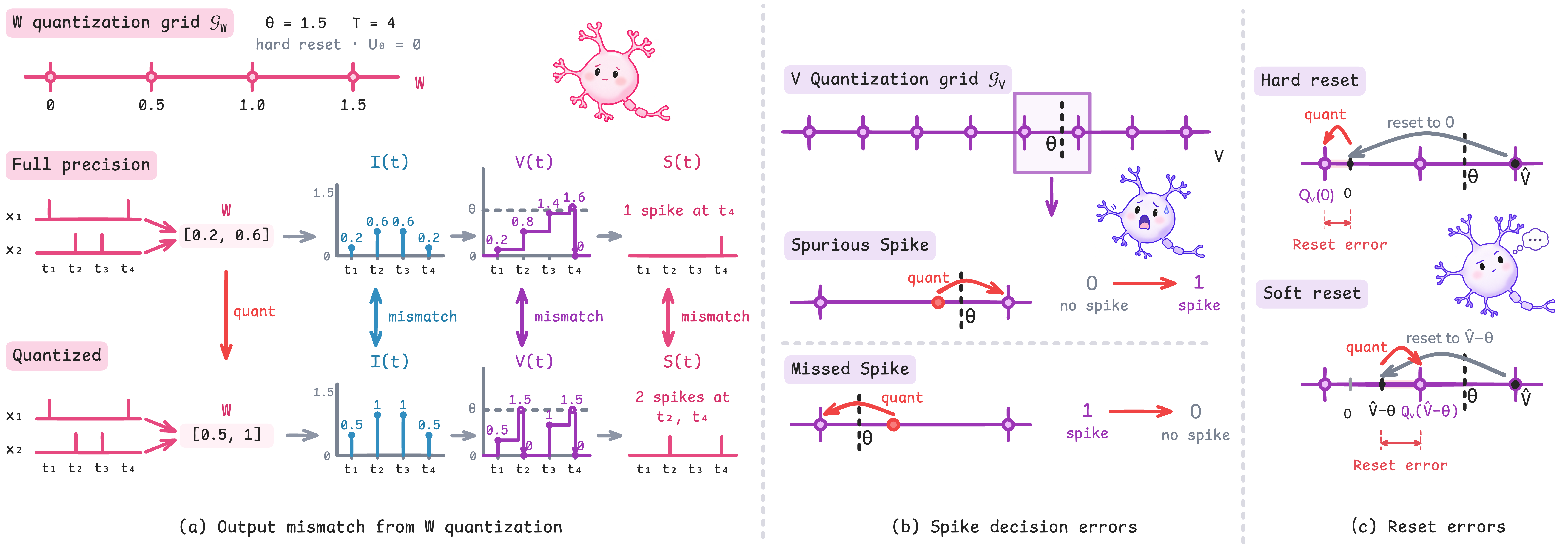}
\caption{Spike errors caused by quantization. (a) Mismatch of the input current, membrane potential, and spikes after weight quantization. (b) Spurious and missed spikes caused by rounding the membrane potential. (c) Reset errors under hard and soft reset.}
\label{fig:problem}
\end{figure*}

\section{Preliminaries}
\label{sec:preliminaries}
\paragraph{LIF Neuron Model.}
For a leaky integrate-and-fire (LIF) neuron with input current $I_t$, let $V_t$ be the membrane potential before firing, $S_t\in\{0,1\}$ the spike, and $U_t$ the state after reset. With time constant $\tau$, threshold $\theta>0$, and $U_0=0$, the dynamics are
\begin{equation}
\begin{aligned}
V_t&=\left(1-\frac1\tau\right)U_{t-1}+\frac{I_t}{\tau},\\
S_t&=\mathbf{1}[V_t\geq\theta],\\
U_t&=
\begin{cases}
V_t(1-S_t), & \text{hard reset},\\
V_t-\theta S_t, & \text{soft reset}.
\end{cases}
\end{aligned}
\label{eq:dynamics}
\end{equation}
Over $T$ steps, a neuron fires $N=\sum_{t}S_t$ spikes.

\paragraph{Weight Quantization.}
With $b$ bits, each weight $w$ is replaced by one of $2^b$ evenly spaced levels, which form the weight grid $\mathcal G_W$. Each level is a multiple of the step $\Delta_w>0$, the distance between neighboring levels. A weight at level $q$ becomes $\widehat w=q\,\Delta_w$, where the integer $q$ takes $2^b$ consecutive values, such as $\{-2,-1,0,1\}$ for $b=2$. Each output channel has its own $\Delta_w$ and range of $q$, so quantizing a weight means selecting its level $q$.

\paragraph{Membrane Quantization.}
With $b$ bits, the membrane potential is rounded by $Q_v$ to one of $K=2^b$ evenly spaced levels with step $\Delta_v$, which form the membrane grid $\mathcal G_V$. Post-reset quantization rounds the state after reset, $\widehat U_t=Q_v(U_t)$, so firing is still decided by the full-precision $V_t$. Pre-fire quantization instead rounds the membrane potential before firing, $\widehat V_t=Q_v(V_t)$, and uses $\widehat V_t$ for both firing and reset.

\section{Method}
\label{sec:method}
SpikeQ uses spikes to guide quantization in two stages. The weight stage selects the level of each weight by simulating the downstream neurons, and the membrane stage builds the membrane grid around the threshold and the reset rule.

\subsection{Weight Quantization}
\label{sec:weight-method}
Quantizing the weights changes the input current, the membrane potential, and the spikes of the downstream neuron (\cref{fig:problem}(a)). Existing methods select the weight levels that minimize the error of the input current or the membrane potential, yet the resulting spikes can still differ from the full-precision ones (\cref{fig:w-method}). SpikeQ instead takes the spikes as the optimization target and selects the weight levels that minimize the spike error.

\begin{figure}[t]
\centering
\includegraphics[width=0.9\columnwidth]{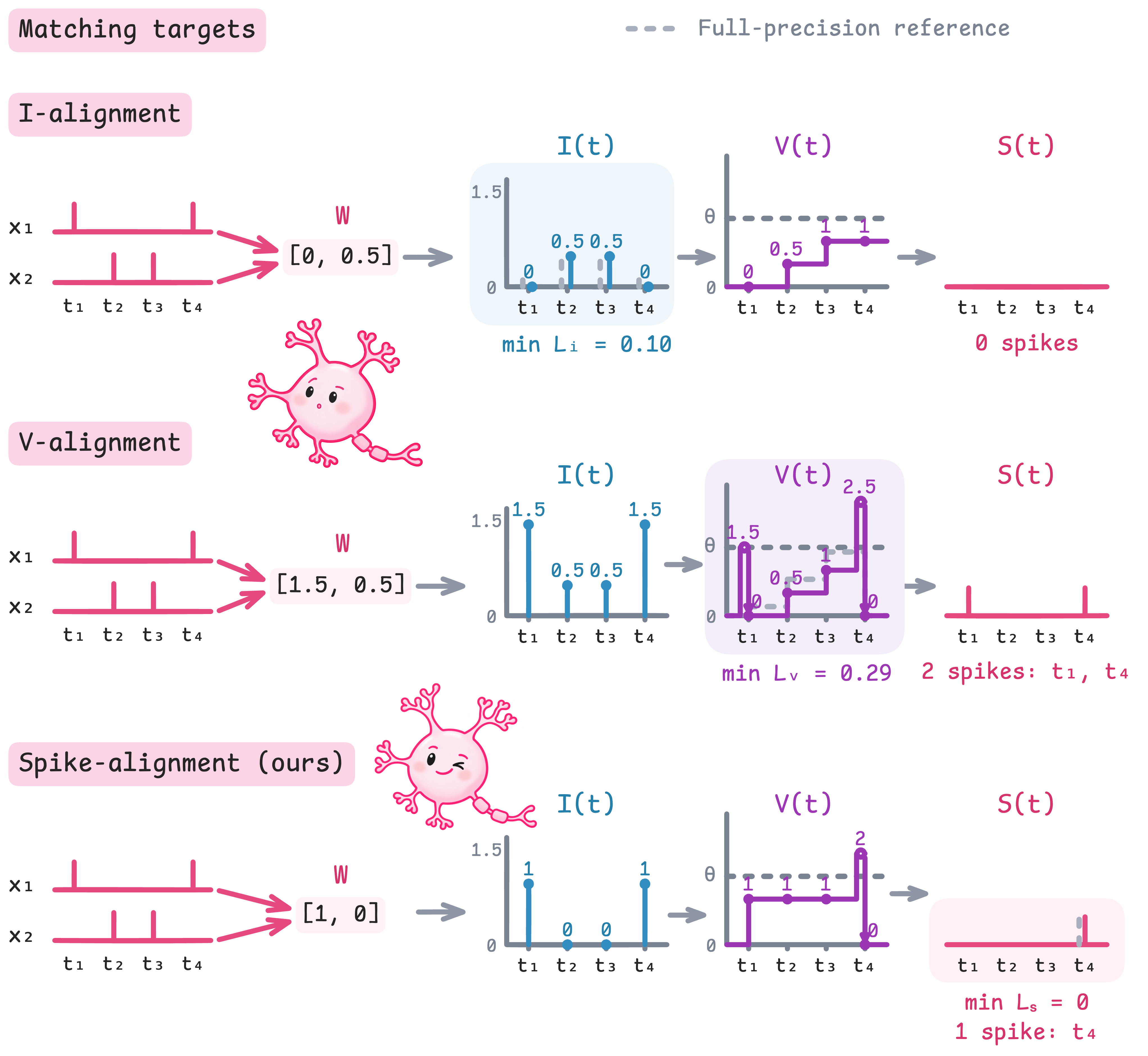}
\caption{Weight rounding for the neuron of \cref{fig:problem}(a). Each row selects the weight levels with the smallest error of the input current (top), membrane potential (middle), or spikes (bottom). Only the spike error recovers the full-precision spike.}
\label{fig:w-method}
\end{figure}

\paragraph{Spike-count objective.}
In layer $\ell$, the input spikes $\widehat{\mathbf S}^{\ell-1}_t$ from layer $\ell-1$ pass through the quantized weights $\widehat W^{\ell}$ and the LIF neurons of \cref{eq:dynamics}, and $\widehat N_i$ counts the spikes of neuron $i$ over the $T$ steps. SpikeQ compares $\widehat N_i$ with two reference counts. The full-precision count $N_i^{\mathrm{FP}}$ comes from the original network, in which no layer is quantized. Matching it directly ignores that the earlier layers have already been quantized and have changed the input spikes $\widehat{\mathbf S}^{\ell-1}$ of layer $\ell$. We therefore also match the local count $N_i^{\mathrm{loc}}$, produced by the full-precision weights $W^{\ell}$ on this changed input $\widehat{\mathbf S}^{\ell-1}$. The spike-count loss is a weighted sum of the errors to the two counts, with $\rho\in[0,1]$,
\begin{equation}
\mathcal L_{\mathrm{count}}
=\rho\sum_i\bigl|\widehat N_i-N_i^{\mathrm{FP}}\bigr|
+(1-\rho)\sum_i\bigl|\widehat N_i-N_i^{\mathrm{loc}}\bigr|.
\label{eq:count-loss}
\end{equation}

\paragraph{Margin term.}
Changing a weight by one level shifts the membrane potential but may not move it across $\theta$. The spike count then does not change, and neither does $\mathcal L_{\mathrm{count}}$. In this case, we choose the level that brings the wrong spikes closest to being corrected. Since each wrong spike occurs at a specific step $t$, we compare spikes step by step with the same two references as the counts, the full-precision spikes $S^{\mathrm{FP}}_{i,t}$ and the local spikes $S^{\mathrm{loc}}_{i,t}$ of neuron $i$. A neuron fires when $\widehat V_{i,t}\ge\theta$, so $|\widehat V_{i,t}-\theta|$ is how much its membrane potential must change to correct a wrong spike. The margin sums this distance over all wrong spikes,
\begin{equation}
M^{\mathrm{FP}}
=\sum_{i,t}\mathbf{1}\bigl[\widehat S_{i,t}\neq S^{\mathrm{FP}}_{i,t}\bigr]
\bigl|\widehat V_{i,t}-\theta\bigr|,
\label{eq:margin}
\end{equation}
and $M^{\mathrm{loc}}$ is defined in the same way with $S^{\mathrm{loc}}_{i,t}$. The two margins are weighted by $\rho$ as in \cref{eq:count-loss}, $\mathcal L_{\mathrm{margin}}=\rho M^{\mathrm{FP}}+(1-\rho)M^{\mathrm{loc}}$.

\paragraph{Population term.}
Because $\mathcal L_{\mathrm{count}}$ takes the absolute value of each error, an extra spike and a missed spike give the same error. It thus cannot tell whether the layer fires too many or too few spikes in total. In either case, the next layer receives a wrong number of input spikes. We therefore add a population term, $\mathcal L_{\mathrm{pop}}=\sum_i(\widehat N_i-\bar N_i)$, which keeps the total spike count of the layer equal to that of the references. Here $\bar N_i=\rho N_i^{\mathrm{FP}}+(1-\rho)N_i^{\mathrm{loc}}$ weights the two references by $\rho$ as in \cref{eq:count-loss}.

The complete objective sums the three terms,
\begin{equation}
\mathcal L=\mathcal L_{\mathrm{count}}
+\lambda\mathcal L_{\mathrm{margin}}+\mu \mathcal L_{\mathrm{pop}},
\label{eq:weight-loss}
\end{equation}
where $\lambda$ is a fixed weight and $\mu$ is adjusted during the search. Since $\mathcal L_{\mathrm{pop}}$ grows with the number of spikes, a larger $\mu$ favors levels with fewer spikes.

\begin{algorithm}[t]
\caption{Weight quantization of SpikeQ}
\label{alg:weight}
\small
\begin{algorithmic}[1]
\Require $W$, calibration set, $b$
\For{each layer $\ell$}
    \State Capture $\widehat{\mathbf S}^{\ell-1}$; fix $\mathcal G_W$ per channel
    \State $\mathcal L\gets$ Eq.~\eqref{eq:weight-loss}
    \State Initialize from $W^{\ell}$; $\mu\gets0$
    \For{each input $j$, largest current first} \Comment{First phase}
        \State $q_j\gets\argmin_{q\in\mathcal A_j}\mathbb E_{\varepsilon}[\mathcal L(q_j=q)]$
    \EndFor
    \For{each sweep} \Comment{Second phase}
        \State Keep $\pm1$ level changes that reduce $\mathcal L$
        \State Restore and freeze channels whose held-out loss rises
        \State Update $\mu$
    \EndFor
    \State Set $\widehat w\gets q\,\Delta_w$ for every weight of $W^{\ell}$
\EndFor
\State \Return $\widehat W$
\end{algorithmic}
\end{algorithm}

\paragraph{Discrete search.}
Since $\mathcal L$ has no useful gradient through the spikes, we search the levels directly (Algorithm~\ref{alg:weight}). The best rounding of one weight depends on how the other weights of the neuron are rounded, and trying all combinations is infeasible. The first phase therefore decides the weights one input at a time, for all neurons at once, starting with the inputs that contribute the most current and thus influence the spikes most. Since the best level for the spikes can lie beyond the two nearest levels, each weight $w_j$ chooses from its four nearest levels $\mathcal A_j=\{\lfloor\bar q\rfloor-1,\lfloor\bar q\rfloor,\lceil\bar q\rceil,\lceil\bar q\rceil+1\}$, where $\bar q=w_j/\Delta_w$ is its unrounded level. The choice takes the weights already decided into account, for example by rounding down when earlier weights have made the input current too large. The weights not yet decided are modeled by random rounding errors $\varepsilon$ drawn uniformly from $[-\Delta_w/2,\Delta_w/2]$, and each candidate is scored by its average loss over several draws.

The first phase only guesses the rounding of later weights, so some early choices may no longer be the best once all weights are fixed. The second phase therefore revisits each weight over several sweeps and keeps a change of one level up or down if it lowers $\mathcal L$. To avoid overfitting the small calibration set, part of it is held out, and a channel whose held-out loss rises after a sweep is restored to its previous levels and frozen. After each sweep, $\mu$ rises when the layer fires too many spikes and falls when it fires too few (\cref{sec:supp-impl}). We keep $|\mu|<1$, below the change of $\mathcal L_{\mathrm{count}}$ when one spike of a neuron is corrected, so that balancing the layer total never outweighs correcting a neuron's spikes.

\subsection{Membrane Quantization}
\label{sec:membrane-method}
Most existing methods use post-reset quantization, whereas neuromorphic chips store the membrane potential at a fixed bit width and compare the stored value with the threshold, so it is rounded before firing. SpikeQ targets this pre-fire quantization, where the rounded membrane potential decides both the spike and the post-reset state. Existing grids leave $\theta$ inside a rounding interval, so rounding can push a membrane potential near $\theta$ across it and cause a spurious or missed spike (\cref{fig:problem}(b)). The post-reset state is stored for the next step, and if it is not a level, storing it rounds it again and introduces another error (\cref{fig:problem}(c)). We therefore build the membrane grid around the threshold and the reset rule (\cref{fig:v-method}).

\begin{figure}[t]
\centering
\includegraphics[width=0.7\columnwidth]{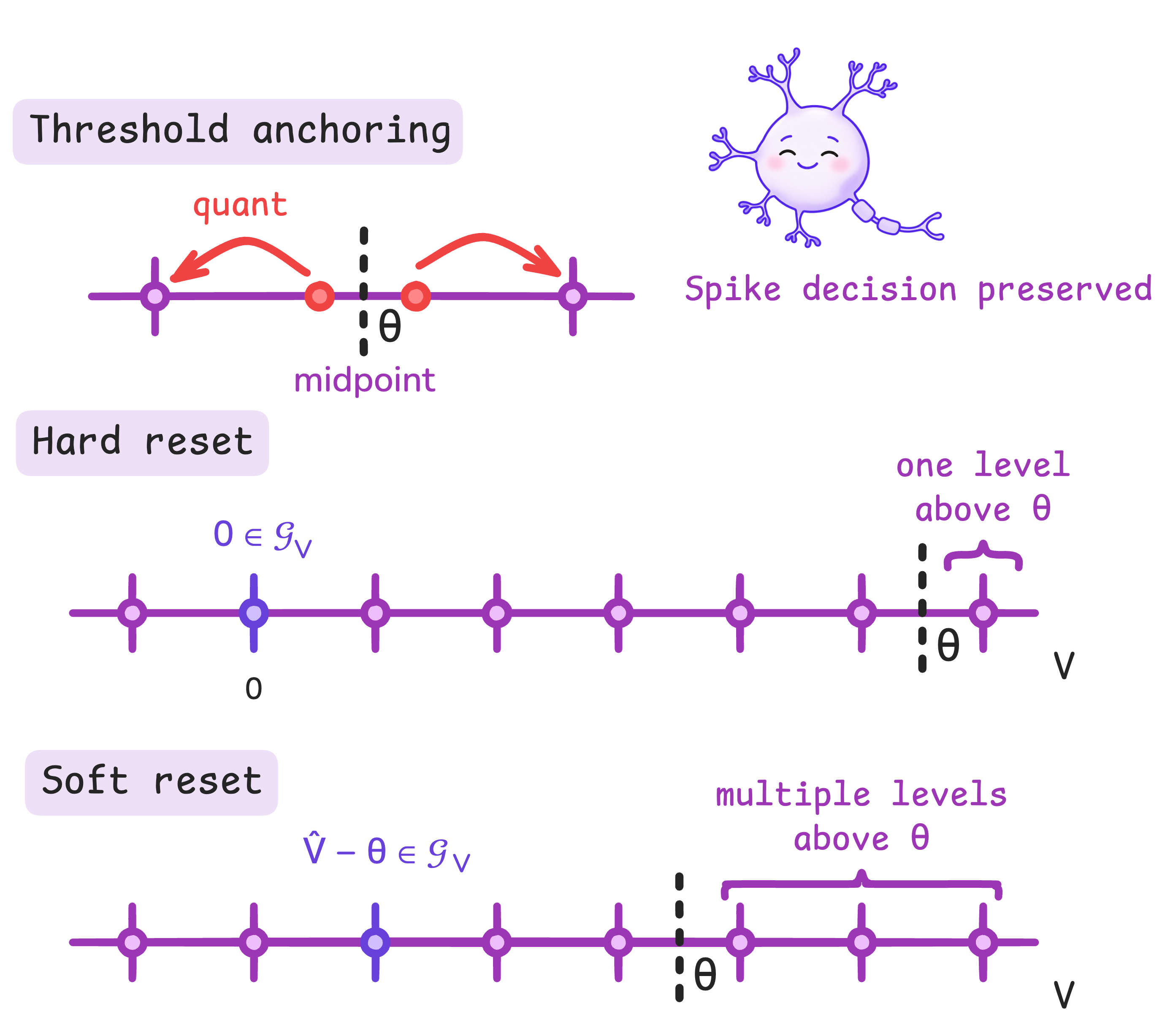}
\caption{Membrane grid of SpikeQ, with threshold anchoring (top), one level above $\theta$ under hard reset (middle), and multiple levels above $\theta$ under soft reset (bottom). The value after reset stays on the grid.}
\label{fig:v-method}
\end{figure}

\paragraph{Threshold anchoring.}
We place $\theta$ midway between two adjacent levels, which makes it the rounding boundary between them (\cref{fig:v-method}, top). With $n$ of the $K$ levels below $\theta$, the levels of $\mathcal G_V$ are
\begin{equation}
g_j=\theta+\bigl(j-n+\tfrac12\bigr)\Delta_v,
\quad j=0,\ldots,K-1,
\label{eq:anchored-grid}
\end{equation}
with step $\Delta_v>0$ and $n\in\{1,\ldots,K-1\}$. The quantizer $Q_v$ rounds to the nearest level, so it never moves a membrane potential across the rounding boundary $\theta$. A membrane potential exactly at $\theta$ is rounded up and still fires, and values outside the grid are clipped to its endpoints, which lie on either side of $\theta$. Hence $\mathbf{1}[Q_v(V)\geq\theta]=\mathbf{1}[V\geq\theta]$ for any $\Delta_v$ and $n$, so rounding the membrane potential never changes the spike at the current step.

\paragraph{Reset-aware level allocation.}
Under hard reset, the post-reset state is $0$ after every spike, so we make $0$ a level and keep only one level above $\theta$, that is, $n=K-1$ (\cref{fig:v-method}, middle). Under soft reset, it is $\widehat V_t-\theta$, which carries the part above $\theta$ to the next step, so several levels above $\theta$ are needed and $n$ is calibrated. We make $\theta$ a whole number $m$ of steps, so that subtracting $\theta$ moves each level $g_j$ above $\theta$ down to the level $g_{j-m}$ (\cref{fig:v-method}, bottom). The two rules give
\begin{equation}
\Delta_v=
\begin{cases}
\theta/(m+\tfrac12) & \text{(hard reset)},\\
\theta/m & \text{(soft reset)},
\end{cases}
\label{eq:reset-grid}
\end{equation}
where $m$ is the number of levels between $0$ and $\theta$, with $m\in\{0,\ldots,K-2\}$ under hard reset and $m\in\{1,\ldots,n\}$ under soft reset, which keeps $g_{j-m}$ on the grid. In both cases, the post-reset state is a level (\cref{sec:supp-grid}).

\paragraph{Calibration by spike errors.}
The grid is now set by $m$ in \cref{eq:reset-grid} and, under soft reset, by the number $n$ of levels below $\theta$. A smaller step rounds more finely but clips more values outside the grid, and the best balance differs between layers. We try all values and compare the grids by the spikes they change, not by the MSE of membrane potentials, which also counts rounding errors that change no spike. For each layer, we record the input currents of its neurons with the earlier layers already quantized. On these currents, we simulate each neuron with a full-precision membrane potential, giving spikes $S_{i,t}$, and with each candidate grid, giving spikes $\widehat S_{i,t}$, and keep the grid with the smallest
\begin{equation}
\mathcal L_v(m,n)=\sum_{i,t}\bigl|\widehat S_{i,t}-S_{i,t}\bigr|.
\label{eq:membrane-calibration}
\end{equation}

\section{Experiments}
\label{sec:experiments}

\subsection{Experimental Setup}
\label{sec:setup}
\paragraph{Benchmarks.}
We evaluate SNNs with $T=4$ time steps on seven benchmarks. For image classification, we use PreAct-ResNet-18~\cite{he2016identity} with soft reset on CIFAR-10 and CIFAR-100~\cite{krizhevsky2009learning}, and SEW-ResNet-18~\cite{fang2021deep} and the Spike-driven Transformer 8-384 (SDT)~\cite{yao2023spike} on ImageNet~\cite{deng2009imagenet}. For semantic segmentation, we use the 55M SDT-V2~\cite{yao2024spike} with a spiking FPN decoder on PASCAL VOC 2012~\cite{everingham2010pascal} and ADE20K~\cite{zhou2017scene}. For object detection, we use EMS-YOLO~\cite{su2023deep} with a ResNet-34 backbone on COCO~\cite{lin2014microsoft}. All other models use hard reset.

\paragraph{Baselines.}
For weights, we compare with round-to-nearest (RTN), FlowQ~\cite{flowq}, the second-order methods GPTQ~\cite{frantar2023gptq}, SBC~\cite{shi2026sbc}, and GPTAQ~\cite{li2025gptaq}, and the reconstruction methods BRECQ~\cite{li2021brecq} and PTQ4SNN~\cite{xie2026ptq4snn}. For membrane potentials, we compare with MinMax~\cite{jacob2018quantization} and MSE~\cite{nagel2021whitepaper} range calibration, the membrane quantization of FlowQ (FlowQ-M), the threshold-centered grid of SQUAT~\cite{venkatesh2024squat}, and the membrane quantization of PTQ4SNN (PTQ4SNN-M). +MPBA adds the mixed-precision bit allocation of PTQ4SNN.

\paragraph{Protocol.}
Most existing methods use post-reset quantization, so we also evaluate all membrane quantization methods under it. W$b$ and V$b$ quantize only the weights or only the membrane potentials to $b$ bits, and W$b$V$b$ quantizes both. Weights are calibrated on 1{,}024 unlabeled training images and membrane grids on 256. Further details and 3-bit results are in \cref{sec:supp-impl,sec:supp-3bit}.

\subsection{Weight Quantization}
\label{sec:exp-w}
\begin{table}[t]
\centering
\footnotesize
\setlength{\aboverulesep}{0pt}
\setlength{\belowrulesep}{0pt}
\renewcommand{\arraystretch}{1.12}
\setlength{\tabcolsep}{1.69pt}
\begin{tabular}{l*{8}{c}}
\toprule
\multirow{2}{*}{Method} & \multicolumn{2}{c}{CIFAR-10} & \multicolumn{2}{c}{CIFAR-100} & \multicolumn{2}{c}{\makecell{ImageNet\\[-1pt]{\scriptsize SEW-ResNet-18}}} & \multicolumn{2}{c}{\makecell{ImageNet\\[-1pt]{\scriptsize SDT-8-384}}} \\
\cmidrule(lr){2-3}\cmidrule(lr){4-5}\cmidrule(lr){6-7}\cmidrule(lr){8-9}
 & W4 & W2 & W4 & W2 & W4 & W2 & W4 & W2 \\
\midrule
FP & \multicolumn{2}{c}{92.60} & \multicolumn{2}{c}{73.78} & \multicolumn{2}{c}{62.89} & \multicolumn{2}{c}{72.04} \\
\midrule
RTN & 92.25 & 17.59 & 70.97 & 1.00 & 36.21 & 0.10 & 40.79 & 0.12 \\
FlowQ~\cite{flowq} & 91.96 & 10.00 & 58.59 & 1.07 & 28.12 & 0.10 & 24.19 & 0.17 \\
GPTQ~\cite{frantar2023gptq} & 92.09 & 89.53 & 71.48 & 50.49 & 61.59 & 13.90 & 68.65 & 4.42 \\
SBC~\cite{shi2026sbc} & \tbest{92.39} & 89.26 & 71.96 & 51.17 & 61.63 & 26.93 & 68.57 & 5.36 \\
GPTAQ~\cite{li2025gptaq} & 92.12 & 90.17 & \tsecond{73.10} & 59.02 & \tsecond{62.02} & 26.46 & \tsecond{69.89} & 2.51 \\
BRECQ~\cite{li2021brecq} & 92.04 & \tsecond{90.32} & 72.01 & 63.76 & 60.95 & 39.56 & 54.14 & 21.83 \\
PTQ4SNN~\cite{xie2026ptq4snn} & 91.89 & \tbest{90.36} & 72.14 & \tsecond{64.19} & 60.86 & \tsecond{40.49} & 54.97 & \tsecond{22.21} \\
\cellcolor{oursw}SpikeQ (Ours) & \cellcolor{oursw}\tsecond{92.38} & \cellcolor{oursw}90.28 & \cellcolor{oursw}\tbest{73.27} & \cellcolor{oursw}\tbest{65.72} & \cellcolor{oursw}\tbest{62.05} & \cellcolor{oursw}\tbest{43.70} & \cellcolor{oursw}\tbest{70.33} & \cellcolor{oursw}\tbest{45.76} \\
\bottomrule
\end{tabular}
\caption{Weight quantization (W$b$) on image classification with full-precision membrane potentials. CIFAR models are PreAct-ResNet-18. Top-1 accuracy (\%). Best in bold and second best underlined.}
\label{tab:cls-w}
\end{table}

\begin{table}[t]
\centering
\footnotesize
\setlength{\aboverulesep}{0pt}
\setlength{\belowrulesep}{0pt}
\renewcommand{\arraystretch}{1.12}
\setlength{\tabcolsep}{11.47pt}
\begin{tabular}{l*{4}{c}}
\toprule
\multirow{2}{*}{Method} & \multicolumn{2}{c}{VOC 2012} & \multicolumn{2}{c}{ADE20K} \\
\cmidrule(lr){2-3}\cmidrule(lr){4-5}
 & W4 & W2 & W4 & W2 \\
\midrule
FP & \multicolumn{2}{c}{70.86} & \multicolumn{2}{c}{34.91} \\
\midrule
RTN & 68.23 & 3.49 & 32.07 & 0.12 \\
FlowQ~\cite{flowq} & 60.92 & 3.49 & 24.40 & 0.05 \\
GPTQ~\cite{frantar2023gptq} & \tbest{70.64} & 52.83 & \tsecond{34.46} & 16.42 \\
SBC~\cite{shi2026sbc} & 70.25 & 56.14 & 34.22 & \tsecond{19.53} \\
GPTAQ~\cite{li2025gptaq} & 69.96 & 47.61 & 33.54 & 15.14 \\
BRECQ~\cite{li2021brecq} & 67.49 & \tsecond{62.98} & 29.85 & 19.28 \\
PTQ4SNN~\cite{xie2026ptq4snn} & 67.53 & 62.05 & 29.84 & 18.58 \\
\cellcolor{oursw}SpikeQ (Ours) & \cellcolor{oursw}\tsecond{70.36} & \cellcolor{oursw}\tbest{65.74} & \cellcolor{oursw}\tbest{34.50} & \cellcolor{oursw}\tbest{29.32} \\
\bottomrule
\end{tabular}
\caption{Weight quantization (W$b$) of SDT-V2 on semantic segmentation with full-precision membrane potentials. mIoU (\%). Notation as in \cref{tab:cls-w}.}
\label{tab:seg-w}
\end{table}

\begin{table}[t]
\centering
\footnotesize
\setlength{\aboverulesep}{0pt}
\setlength{\belowrulesep}{0pt}
\renewcommand{\arraystretch}{1.12}
\setlength{\tabcolsep}{11.47pt}
\begin{tabular}{l*{4}{c}}
\toprule
\multirow{2}{*}{Method} & \multicolumn{2}{c}{AP} & \multicolumn{2}{c}{AP$_{50}$} \\
\cmidrule(lr){2-3}\cmidrule(lr){4-5}
 & W4 & W2 & W4 & W2 \\
\midrule
FP & \multicolumn{2}{c}{31.38} & \multicolumn{2}{c}{50.86} \\
\midrule
RTN & 22.59 & 0.00 & 44.57 & 0.00 \\
FlowQ~\cite{flowq} & 12.98 & 0.00 & 30.40 & 0.00 \\
GPTQ~\cite{frantar2023gptq} & 29.97 & 7.32 & \tbest{50.25} & 15.67 \\
SBC~\cite{shi2026sbc} & \tsecond{30.03} & 10.02 & \tsecond{50.16} & 20.12 \\
GPTAQ~\cite{li2025gptaq} & 29.12 & 1.46 & 49.58 & 3.62 \\
BRECQ~\cite{li2021brecq} & 25.40 & \tsecond{14.82} & 45.80 & 31.72 \\
PTQ4SNN~\cite{xie2026ptq4snn} & 25.31 & 14.62 & 45.89 & \tsecond{32.01} \\
\cellcolor{oursw}SpikeQ (Ours) & \cellcolor{oursw}\tbest{30.23} & \cellcolor{oursw}\tbest{21.30} & \cellcolor{oursw}50.10 & \cellcolor{oursw}\tbest{39.91} \\
\bottomrule
\end{tabular}
\caption{Weight quantization (W$b$) of EMS-YOLO on COCO with full-precision membrane potentials. AP and AP$_{50}$ (\%). Notation as in \cref{tab:cls-w}.}
\label{tab:coco}
\end{table}

\begin{table}[t]
\centering
\footnotesize
\setlength{\aboverulesep}{0pt}
\setlength{\belowrulesep}{0pt}
\renewcommand{\arraystretch}{1.12}
\setlength{\tabcolsep}{1.61pt}
\begin{tabular}{l*{4}{c}}
\toprule
\multirow{2}{*}{Method} & \multicolumn{2}{c}{CIFAR-10} & \multicolumn{2}{c}{CIFAR-100} \\
\cmidrule(lr){2-3}\cmidrule(lr){4-5}
 & V4 & V2 & V4 & V2 \\
\midrule
FP & \multicolumn{2}{c}{92.60} & \multicolumn{2}{c}{73.78} \\
\midrule
MinMax~\cite{jacob2018quantization} & 68.99/91.92 & 10.00/20.78 & 43.32/72.76 & 0.78/16.00 \\
MSE~\cite{nagel2021whitepaper} & 81.61/92.47 & 10.00/91.10 & 69.62/73.53 & 1.00/56.27 \\
FlowQ-M~\cite{flowq} & 91.50/92.51 & 10.00/92.26 & 66.69/\tsecond{73.91} & 1.00/71.56 \\
SQUAT~\cite{venkatesh2024squat} & 92.06/92.37 & 10.01/90.27 & \tsecond{71.00}/73.88 & 1.00/66.52 \\
PTQ4SNN-M$^\dagger$~\cite{xie2026ptq4snn} & 90.90/92.53 & 10.00/91.89 & 63.55/73.71 & 1.00/70.23 \\
\quad+MPBA$^\dagger$~\cite{xie2026ptq4snn} & 63.32/92.54 & 10.00/92.03$^\ddagger$ & 68.09/73.74 & 1.00/70.62$^\ddagger$ \\
\cellcolor{oursv}SpikeQ (Ours) & \cellcolor{oursv}\tsecond{92.68}/\tsecond{92.55} & \cellcolor{oursv}\tbest{91.36}/\tsecond{92.48} & \cellcolor{oursv}\tbest{73.27}/\tbest{74.04} & \cellcolor{oursv}\tsecond{68.60}/\tsecond{72.96} \\
\cellcolor{oursv}SpikeQ$^\dagger$ (Ours) & \cellcolor{oursv}\tbest{92.80}/\tbest{92.66} & \cellcolor{oursv}\tsecond{91.31}/\tbest{92.76} & \cellcolor{oursv}\tbest{73.27}/73.37 & \cellcolor{oursv}\tbest{68.65}/\tbest{73.44} \\
\midrule[\heavyrulewidth]
\multirow{2}{*}{Method} & \multicolumn{2}{c}{\makecell{ImageNet\\[-1pt]{\scriptsize SEW-ResNet-18}}} & \multicolumn{2}{c}{\makecell{ImageNet\\[-1pt]{\scriptsize SDT-8-384}}} \\
\cmidrule(lr){2-3}\cmidrule(lr){4-5}
 & V4 & V2 & V4 & V2 \\
\midrule
FP & \multicolumn{2}{c}{62.89} & \multicolumn{2}{c}{72.04} \\
\midrule
MinMax~\cite{jacob2018quantization} & 0.15/1.73 & 0.11/0.11 & 7.87/66.39 & 0.10/19.58 \\
MSE~\cite{nagel2021whitepaper} & 0.10/9.29 & 0.10/0.10 & 0.47/64.33 & 0.10/5.90 \\
FlowQ-M~\cite{flowq} & 1.15/28.65 & 0.11/3.73 & 51.50/71.05 & 0.84/62.46 \\
SQUAT~\cite{venkatesh2024squat} & 4.46/60.64 & 0.08/4.06 & 0.38/71.71 & 0.10/55.89 \\
PTQ4SNN-M$^\dagger$~\cite{xie2026ptq4snn} & 13.98/47.95 & 0.12/14.30 & 45.25/70.80 & 0.14/62.26 \\
\quad+MPBA$^\dagger$~\cite{xie2026ptq4snn} & 15.83/48.11 & 0.16/14.50$^\ddagger$ & 35.61/70.21 & 0.25/62.50$^\ddagger$ \\
\cellcolor{oursv}SpikeQ (Ours) & \cellcolor{oursv}\tsecond{61.39}/\tsecond{61.40} & \cellcolor{oursv}\tbest{48.32}/\tbest{48.29} & \cellcolor{oursv}\tsecond{71.97}/\tsecond{71.97} & \cellcolor{oursv}\tsecond{67.87}/\tsecond{67.88} \\
\cellcolor{oursv}SpikeQ$^\dagger$ (Ours) & \cellcolor{oursv}\tbest{62.71}/\tbest{62.82} & \cellcolor{oursv}\tsecond{35.94}/\tsecond{35.91} & \cellcolor{oursv}\tbest{71.99}/\tbest{72.03} & \cellcolor{oursv}\tbest{68.36}/\tbest{68.47} \\
\bottomrule
\end{tabular}
\caption{Membrane quantization (V$b$) on image classification with full-precision weights. CIFAR models are PreAct-ResNet-18. Top-1 accuracy (\%). Each cell gives the pre-fire/post-reset result, ranked separately with the best in bold and second best underlined. $^\dagger$Channel-wise. $^\ddagger$2.3 bits on average, not ranked.}
\label{tab:cls-v}
\end{table}

\Cref{tab:cls-w,tab:seg-w,tab:coco} compare SpikeQ with existing weight quantization methods. At W4, it stays within 0.28 points of the best method on every benchmark, whereas at W2 it is the best on six of the seven benchmarks and more than doubles the accuracy of the strongest baseline on SDT (45.76\% vs.\ 22.21\%). The difference comes from the step size. With a fine 4-bit step, the rounding error is small and flips only neurons whose membrane potential is close to the threshold, so minimizing the error of the current or membrane potential also keeps flipped spikes few. With a coarse 2-bit step, the error cannot be made small, and what matters is where it falls. An error flips no spike in a neuron far from the threshold but flips one near it, so the rounding with the smallest error of the current or membrane potential can still flip spikes. SpikeQ shares its clipping with BRECQ and PTQ4SNN but selects levels by the spikes they change, so it can place the error where it flips none. SpikeQ thus flips the fewest spikes at W2 on every benchmark, and accuracy generally falls as flipped spikes grow (\cref{fig:flip-acc}(a) and \cref{fig:supp-flip-2bit,fig:supp-flip-4bit}).

The gain also grows with the complexity of the task, from CIFAR-10 to CIFAR-100 to ImageNet for ResNet-18 and from VOC to ADE20K for SDT-V2. A complex task separates many classes by fine differences in the features, which leaves the network less redundancy to absorb flipped spikes. At W2, the strongest baseline thus loses about twice as much accuracy on ADE20K as on VOC, so reducing flipped spikes helps most on such tasks. On CIFAR-10, whose ten classes differ coarsely, the network has enough redundancy to tolerate the flipped spikes of every strong method, so at W2 the four strongest lie within 0.19 points.

\subsection{Membrane Quantization}
\label{sec:exp-v}
\begin{table}[t]
\centering
\footnotesize
\setlength{\aboverulesep}{0pt}
\setlength{\belowrulesep}{0pt}
\renewcommand{\arraystretch}{1.12}
\setlength{\tabcolsep}{2.01pt}
\begin{tabular}{l*{4}{c}}
\toprule
\multirow{2}{*}{Method} & \multicolumn{2}{c}{VOC 2012} & \multicolumn{2}{c}{ADE20K} \\
\cmidrule(lr){2-3}\cmidrule(lr){4-5}
 & V4 & V2 & V4 & V2 \\
\midrule
FP & \multicolumn{2}{c}{70.86} & \multicolumn{2}{c}{34.91} \\
\midrule
MinMax~\cite{jacob2018quantization} & 0.05/62.40 & 3.49/50.10 & 0.17/24.87 & 0.12/15.84 \\
MSE~\cite{nagel2021whitepaper} & 3.49/69.10 & 3.49/41.22 & 0.27/33.62 & 0.12/24.01 \\
FlowQ-M~\cite{flowq} & 16.88/70.29 & 3.49/68.52 & 11.07/34.02 & 0.12/33.41 \\
SQUAT~\cite{venkatesh2024squat} & 2.91/70.61 & 0.20/61.47 & 1.59/\tbest{34.72} & 0.13/31.69 \\
PTQ4SNN-M$^\dagger$~\cite{xie2026ptq4snn} & 61.14/\tsecond{70.76} & 3.49/\tsecond{70.12} & 27.32/34.61 & 0.16/\tbest{34.29} \\
\quad+MPBA$^\dagger$~\cite{xie2026ptq4snn} & 57.16/70.70 & 3.49/69.42$^\ddagger$ & 24.33/\tsecond{34.69} & 0.12/34.27$^\ddagger$ \\
\cellcolor{oursv}SpikeQ (Ours) & \cellcolor{oursv}\tbest{70.65}/70.73 & \cellcolor{oursv}\tsecond{69.71}/69.79 & \cellcolor{oursv}\tbest{34.56}/34.56 & \cellcolor{oursv}\tsecond{33.44}/33.67 \\
\cellcolor{oursv}SpikeQ$^\dagger$ (Ours) & \cellcolor{oursv}\tsecond{70.46}/\tbest{70.85} & \cellcolor{oursv}\tbest{70.61}/\tbest{70.32} & \cellcolor{oursv}\tsecond{34.46}/34.54 & \cellcolor{oursv}\tbest{34.27}/\tsecond{33.94} \\
\bottomrule
\end{tabular}
\caption{Membrane quantization (V$b$) of SDT-V2 on semantic segmentation with full-precision weights. mIoU (\%). Notation as in \cref{tab:cls-v}.}
\label{tab:seg-v}
\end{table}

\begin{table}[t]
\centering
\footnotesize
\setlength{\aboverulesep}{0pt}
\setlength{\belowrulesep}{0pt}
\renewcommand{\arraystretch}{1.12}
\setlength{\tabcolsep}{16.09pt}
\begin{tabular}{l*{2}{c}}
\toprule
Method & V4 & V2 \\
\midrule
FP & \multicolumn{2}{c}{31.38} \\
\midrule
MinMax~\cite{jacob2018quantization} & 0.79/30.97 & 0.00/26.35 \\
MSE~\cite{nagel2021whitepaper} & 0.27/31.03 & 0.00/27.84 \\
FlowQ-M~\cite{flowq} & 28.41/31.20 & 0.13/30.84 \\
SQUAT~\cite{venkatesh2024squat} & 20.30/31.36 & 0.00/30.60 \\
PTQ4SNN-M$^\dagger$~\cite{xie2026ptq4snn} & 28.54/31.27 & 0.32/\tbest{31.28} \\
\quad+MPBA$^\dagger$~\cite{xie2026ptq4snn} & 27.49/\tsecond{31.39} & 0.49/31.13$^\ddagger$ \\
\cellcolor{oursv}SpikeQ (Ours) & \cellcolor{oursv}\tsecond{31.32}/\tbest{31.53} & \cellcolor{oursv}\tsecond{31.14}/\tsecond{31.24} \\
\cellcolor{oursv}SpikeQ$^\dagger$ (Ours) & \cellcolor{oursv}\tbest{31.39}/31.38 & \cellcolor{oursv}\tbest{31.29}/31.13 \\
\bottomrule
\end{tabular}
\caption{Membrane quantization (V$b$) of EMS-YOLO on COCO with full-precision weights. AP (\%). Notation as in \cref{tab:cls-v}.}
\label{tab:coco-v}
\end{table}

\begin{figure}[t]
\centering
\includegraphics[width=\columnwidth]{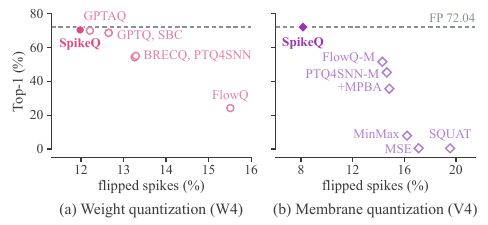}
\caption{Accuracy versus flipped spikes of SDT-8-384 on ImageNet. (a) Weight quantization (W4) with full-precision membrane potentials. (b) Membrane quantization (V4) with full-precision weights. Flipped spikes are counted against the full-precision network over all neurons and time steps.}
\label{fig:flip-acc}
\end{figure}
\paragraph{Pre-fire quantization.}
PTQ4SNN-M selects one scale per channel, so besides the layer-wise SpikeQ, we report SpikeQ$^\dagger$, which selects one grid per channel. Under pre-fire quantization, every existing grid collapses at V2 on all seven benchmarks, whereas SpikeQ retains most of the full-precision accuracy (\cref{tab:cls-v,tab:seg-v,tab:coco-v}). In existing grids, $\theta$ is not a rounding boundary, so membrane potentials between $\theta$ and the nearest one are rounded across $\theta$. This band spans up to half a step, so at 2 bits it flips many spikes in every layer and time step. The smaller step at V4 narrows the band, and existing grids partly recover, but their accuracy still falls as they flip more spikes (\cref{fig:flip-acc}(b)). Threshold anchoring makes $\theta$ itself a rounding boundary and removes the band for any $\Delta_v$, so SpikeQ flips the fewest spikes. SpikeQ$^\dagger$ improves on SpikeQ by up to 0.90 mIoU on segmentation at V2. On SEW-ResNet-18 at V2, however, it falls far behind SpikeQ, as some first-layer channels lose 73\% of their spikes. Their grid has no level between zero and $\theta$, so membrane potentials below $\theta$ are rounded to zero and cannot accumulate, whereas a layer-wise grid must suit all channels and keeps such levels.

\paragraph{Post-reset quantization.}
We also test whether minimizing flipped spikes helps under post-reset quantization, where firing still uses the full-precision $V_t$. Our grids are again more accurate than existing ones on most benchmarks (\cref{tab:cls-v,tab:seg-v,tab:coco-v}). Existing grids set the step without regard to flipped spikes. Under hard reset, for example, the MSE is dominated by negative membrane potentials far from $\theta$, whose rounding errors flip no spike. It thus selects a coarse step that rounds membrane potentials just below $\theta$ to zero, so their neurons miss the spikes they would fire at the next time step. SpikeQ selects the step by the number of flipped spikes and thus ignores these far membrane potentials.

\paragraph{Hard vs.\ soft reset.}
Under hard reset, a calibration free to select the number of levels above $\theta$ still keeps only one, so fixing this number in our grid loses nothing and reduces the search.
On the soft-reset CIFAR models, our pre-fire grids trail our post-reset grids at V2. Pre-fire quantization rounds the membrane potential before $\theta$ is subtracted and therefore also needs levels above $\theta$, whereas post-reset quantization rounds the difference, which lies below $\theta$. With the same number of levels, the pre-fire grid thus spans a wider range and is coarser, and the gap largely closes with the 16 levels of V4.

\definecolor{ecgray}{gray}{0.55}
\providecommand{\ecsubtext}[1]{\ensuremath{_{\mbox{\fontsize{5pt}{5pt}\selectfont#1}}}}
\providecommand{\ecc}[4]{\makebox[#1][r]{#3}\makebox[#2][l]{\ecsubtext{\textcolor{ecgray}{#4}}}}
\providecommand{\echd}[3]{\makebox[#1][c]{#3}\makebox[#2]{}}
\ifdefined\ecnA\else\newlength{\ecnA}\fi
\ifdefined\ecnB\else\newlength{\ecnB}\fi
\ifdefined\ecsA\else\newlength{\ecsA}\fi
\begin{table}[t]
\centering
\footnotesize
\setlength{\aboverulesep}{0pt}
\setlength{\belowrulesep}{0pt}
\renewcommand{\arraystretch}{1.12}
\settowidth{\ecnA}{0.000}\settowidth{\ecnB}{0.00}\settowidth{\ecsA}{\ecsubtext{0.00\texttimes}}
\setlength{\tabcolsep}{1.42pt}
\begin{tabular}{@{}l*{7}{c}@{}}
\toprule
\multirow{2}{*}{Setting} & \multicolumn{2}{c}{CIFAR\hspace*{\ecsA}} & \multicolumn{2}{c}{ImageNet\hspace*{\ecsA}} & \multirow{2}{*}{\echd{\ecnB}{\ecsA}{VOC}} & \multirow{2}{*}{\echd{\ecnB}{\ecsA}{\makecell{ADE\\[-1pt]20K}}} & \multirow{2}{*}{\echd{\ecnB}{\ecsA}{COCO}} \\
\cmidrule(l{0.5em}r{\dimexpr\ecsA+0.5em\relax}){2-3}\cmidrule(l{0.5em}r{\dimexpr\ecsA+0.5em\relax}){4-5}
& \echd{\ecnA}{\ecsA}{10} & \echd{\ecnA}{\ecsA}{100} & \echd{\ecnA}{\ecsA}{SEW} & \echd{\ecnB}{\ecsA}{SDT} & & & \\
\midrule
FP & \ecc{\ecnA}{\ecsA}{92.60}{} & \ecc{\ecnA}{\ecsA}{73.78}{} & \ecc{\ecnA}{\ecsA}{62.89}{} & \ecc{\ecnB}{\ecsA}{72.04}{} & \ecc{\ecnB}{\ecsA}{70.86}{} & \ecc{\ecnB}{\ecsA}{34.91}{} & \ecc{\ecnB}{\ecsA}{31.38}{} \\
W4 & \ecc{\ecnA}{\ecsA}{92.38}{} & \ecc{\ecnA}{\ecsA}{73.27}{} & \ecc{\ecnA}{\ecsA}{61.91}{} & \ecc{\ecnB}{\ecsA}{70.33}{} & \ecc{\ecnB}{\ecsA}{70.36}{} & \ecc{\ecnB}{\ecsA}{34.40}{} & \ecc{\ecnB}{\ecsA}{29.92}{} \\
W4V4 & \ecc{\ecnA}{\ecsA}{92.41}{} & \ecc{\ecnA}{\ecsA}{73.11}{} & \ecc{\ecnA}{\ecsA}{60.30}{} & \ecc{\ecnB}{\ecsA}{70.49}{} & \ecc{\ecnB}{\ecsA}{70.32}{} & \ecc{\ecnB}{\ecsA}{34.12}{} & \ecc{\ecnB}{\ecsA}{29.80}{} \\
W2 & \ecc{\ecnA}{\ecsA}{90.28}{} & \ecc{\ecnA}{\ecsA}{65.72}{} & \ecc{\ecnA}{\ecsA}{43.70}{} & \ecc{\ecnB}{\ecsA}{45.76}{} & \ecc{\ecnB}{\ecsA}{65.74}{} & \ecc{\ecnB}{\ecsA}{29.32}{} & \ecc{\ecnB}{\ecsA}{21.30}{} \\
W2V2 & \ecc{\ecnA}{\ecsA}{88.50}{} & \ecc{\ecnA}{\ecsA}{57.49}{} & \ecc{\ecnA}{\ecsA}{15.85}{} & \ecc{\ecnB}{\ecsA}{36.83}{} & \ecc{\ecnB}{\ecsA}{63.00}{} & \ecc{\ecnB}{\ecsA}{27.18}{} & \ecc{\ecnB}{\ecsA}{21.26}{} \\
\midrule
\multicolumn{8}{@{}l}{Footprint (MB)} \\
FP & \ecc{\ecnA}{\ecsA}{44.8}{} & \ecc{\ecnA}{\ecsA}{44.9}{} & \ecc{\ecnA}{\ecsA}{54.1}{} & \ecc{\ecnB}{\ecsA}{101}{} & \ecc{\ecnB}{\ecsA}{586}{} & \ecc{\ecnB}{\ecsA}{723}{} & \ecc{\ecnB}{\ecsA}{269}{} \\
W4 & \ecc{\ecnA}{\ecsA}{7.52}{6.0\texttimes} & \ecc{\ecnA}{\ecsA}{7.54}{6.0\texttimes} & \ecc{\ecnA}{\ecsA}{15.1}{3.6\texttimes} & \ecc{\ecnB}{\ecsA}{45.4}{2.2\texttimes} & \ecc{\ecnB}{\ecsA}{391}{1.5\texttimes} & \ecc{\ecnB}{\ecsA}{528}{1.4\texttimes} & \ecc{\ecnB}{\ecsA}{156}{1.7\texttimes} \\
W4V4 & \ecc{\ecnA}{\ecsA}{5.66}{7.9\texttimes} & \ecc{\ecnA}{\ecsA}{5.68}{7.9\texttimes} & \ecc{\ecnA}{\ecsA}{6.81}{7.9\texttimes} & \ecc{\ecnB}{\ecsA}{13.0}{7.8\texttimes} & \ecc{\ecnB}{\ecsA}{74.2}{7.9\texttimes} & \ecc{\ecnB}{\ecsA}{91.4}{7.9\texttimes} & \ecc{\ecnB}{\ecsA}{33.8}{7.9\texttimes} \\
W2 & \ecc{\ecnA}{\ecsA}{4.86}{9.2\texttimes} & \ecc{\ecnA}{\ecsA}{4.87}{9.2\texttimes} & \ecc{\ecnA}{\ecsA}{12.3}{4.4\texttimes} & \ecc{\ecnB}{\ecsA}{41.4}{2.4\texttimes} & \ecc{\ecnB}{\ecsA}{377}{1.6\texttimes} & \ecc{\ecnB}{\ecsA}{514}{1.4\texttimes} & \ecc{\ecnB}{\ecsA}{148}{1.8\texttimes} \\
W2V2 & \ecc{\ecnA}{\ecsA}{2.86}{15.6\texttimes} & \ecc{\ecnA}{\ecsA}{2.87}{15.6\texttimes} & \ecc{\ecnA}{\ecsA}{3.44}{15.7\texttimes} & \ecc{\ecnB}{\ecsA}{6.72}{15.1\texttimes} & \ecc{\ecnB}{\ecsA}{37.6}{15.6\texttimes} & \ecc{\ecnB}{\ecsA}{46.2}{15.6\texttimes} & \ecc{\ecnB}{\ecsA}{17.1}{15.7\texttimes} \\
\midrule
\multicolumn{8}{@{}l}{EC (mJ)} \\
FP & \ecc{\ecnA}{\ecsA}{1.56}{} & \ecc{\ecnA}{\ecsA}{0.889}{} & \ecc{\ecnA}{\ecsA}{5.96}{} & \ecc{\ecnB}{\ecsA}{18.4}{} & \ecc{\ecnB}{\ecsA}{421}{} & \ecc{\ecnB}{\ecsA}{556}{} & \ecc{\ecnB}{\ecsA}{154}{} \\
W4V4 & \ecc{\ecnA}{\ecsA}{0.228}{6.8\texttimes} & \ecc{\ecnA}{\ecsA}{0.129}{6.9\texttimes} & \ecc{\ecnA}{\ecsA}{0.875}{6.8\texttimes} & \ecc{\ecnB}{\ecsA}{2.71}{6.8\texttimes} & \ecc{\ecnB}{\ecsA}{62.1}{6.8\texttimes} & \ecc{\ecnB}{\ecsA}{81.6}{6.8\texttimes} & \ecc{\ecnB}{\ecsA}{24.1}{6.4\texttimes} \\
W2V2 & \ecc{\ecnA}{\ecsA}{0.147}{10.6\texttimes} & \ecc{\ecnA}{\ecsA}{0.083}{10.8\texttimes} & \ecc{\ecnA}{\ecsA}{0.576}{10.3\texttimes} & \ecc{\ecnB}{\ecsA}{1.76}{10.5\texttimes} & \ecc{\ecnB}{\ecsA}{41.0}{10.2\texttimes} & \ecc{\ecnB}{\ecsA}{56.1}{9.9\texttimes} & \ecc{\ecnB}{\ecsA}{17.7}{8.7\texttimes} \\
\bottomrule
\end{tabular}
\caption{Joint quantization with SpikeQ, memory footprint, and energy consumption (EC). The W$b$ accuracies repeat \cref{tab:cls-w,tab:seg-w,tab:coco}. Top-1 accuracy, mIoU, and AP (\%). Footprint adds the model size and the memory of membrane potentials for one image (\cref{sec:supp-energy}). Gray subscripts give the reduction from FP.}
\label{tab:wv}
\end{table}

\subsection{Joint Quantization}
\label{sec:exp-wv}
Quantizing only the weights leaves the membrane potentials at 32 bits, so W2 shrinks the memory footprint by only \mbox{1.4--9.2$\times$}, whereas W2V2 shrinks it by \mbox{15.1--15.7$\times$} (\cref{tab:wv}). W4V4 shrinks it by \mbox{7.8--7.9$\times$} and stays within 2.59 points of full precision. W2V2 trades accuracy for this further 2$\times$ compression, and the price grows with the complexity of the task. At 2~bits, both stages still flip spikes, and accuracy drops steeply once flipped spikes accumulate. A complex task has less redundancy to absorb them (\cref{sec:exp-w}) and reaches this drop sooner. Complex tasks thus need 4~bits, whereas 2~bits pay off only on simple tasks under tight memory budgets.

\subsection{Energy Consumption}
\label{sec:exp-energy}
Following prior work~\cite{yan2024reconsidering,yan2026otters}, the energy consumption (EC) of one inference sample is written as
\begin{equation}
\mathrm{EC}=E_{\mathrm{compute}}+E_{\mathrm{memory}}+E_{\mathrm{neuron}}.
\label{eq:ec}
\end{equation}
The compute term covers adding active weights to membrane potentials, the memory term covers reading these weights and transmitting spikes, and the neuron term covers membrane potential updates, threshold comparisons, and resets. Each term multiplies an operation count, which grows with the firing rate, by a per-operation energy, which grows with the bit width (\cref{sec:supp-energy}).

EC falls by \mbox{6.4--6.9$\times$} at W4V4 and \mbox{8.7--10.8$\times$} at W2V2 (\cref{tab:wv}). Unlike the footprint, which shrinks fully only when the membrane potentials are also quantized, EC falls almost entirely through weight quantization. Weight reads, whose cost is proportional to the bit width, dominate the full-precision EC, whereas the bit width of the membrane potentials enters only the neuron term, a small share of EC.

\ifdefined\ablsubw\else\newlength{\ablsubw}\newlength{\ablnumw}\fi
\providecommand{\ablnum}[1]{\makebox[\ablnumw][r]{#1}}
\providecolor{ecgray}{gray}{0.55}
\providecommand{\ablsub}[1]{\makebox[\ablsubw][l]{\textsubscript{\textcolor{ecgray}{#1}}}}
\begin{table}[t]
\centering
\footnotesize
\settowidth{\ablnumw}{00.00}\settowidth{\ablsubw}{\textsubscript{\textminus50.93}}
\setlength{\aboverulesep}{0pt}
\setlength{\belowrulesep}{0pt}
\renewcommand{\arraystretch}{1.12}
{\setlength{\tabcolsep}{1.60pt}%
\begin{tabular}{@{}cl*{4}{c}}
\toprule
& \multirow{2}{*}{Variant} & \multicolumn{2}{c}{CIFAR} & \multicolumn{2}{c}{ImageNet} \\
\cmidrule(lr){3-4}\cmidrule(lr){5-6}
&& 10 & 100 & SEW & SDT \\
\midrule
\multirow{3}{*}{W2} & \cellcolor{oursw}SpikeQ (Ours) & \cellcolor{oursw}\ablnum{90.28}\ablsub{} & \cellcolor{oursw}\ablnum{65.72}\ablsub{} & \cellcolor{oursw}\ablnum{43.70}\ablsub{} & \cellcolor{oursw}\ablnum{45.76}\ablsub{} \\
 & w/o $\mathcal L_{\mathrm{margin}}$ & \ablnum{89.86}\ablsub{\textminus0.42} & \ablnum{64.00}\ablsub{\textminus1.72} & \ablnum{40.25}\ablsub{\textminus3.45} & \ablnum{40.47}\ablsub{\textminus5.29} \\
 & w/o $\mathcal L_{\mathrm{pop}}$ & \ablnum{90.28}\ablsub{0.00} & \ablnum{65.30}\ablsub{\textminus0.42} & \ablnum{42.26}\ablsub{\textminus1.44} & \ablnum{44.88}\ablsub{\textminus0.88} \\
\midrule
\multirow{3}{*}{\makecell{Pre\\V2}} & \cellcolor{oursv}SpikeQ (Ours) & \cellcolor{oursv}\ablnum{91.36}\ablsub{} & \cellcolor{oursv}\ablnum{68.60}\ablsub{} & \cellcolor{oursv}\ablnum{48.32}\ablsub{} & \cellcolor{oursv}\ablnum{67.87}\ablsub{} \\
 & w/o anchoring & \ablnum{83.32}\ablsub{\textminus8.04} & \ablnum{44.92}\ablsub{\textminus23.68} & \ablnum{0.96}\ablsub{\textminus47.36} & \ablnum{17.26}\ablsub{\textminus50.61} \\
 & w/o reset-aware $\Delta_v$ & \ablnum{91.06}\ablsub{\textminus0.30} & \ablnum{68.50}\ablsub{\textminus0.10} & \ablnum{47.74}\ablsub{\textminus0.58} & \ablnum{67.34}\ablsub{\textminus0.53} \\
\midrule
\multirow{2}{*}{\makecell{Post\\V2}} & \cellcolor{oursv}SpikeQ (Ours) & \cellcolor{oursv}\ablnum{92.48}\ablsub{} & \cellcolor{oursv}\ablnum{72.96}\ablsub{} & \cellcolor{oursv}\ablnum{48.29}\ablsub{} & \cellcolor{oursv}\ablnum{67.88}\ablsub{} \\
 & MSE instead of $\mathcal L_v$ & \ablnum{92.49}\ablsub{+0.01} & \ablnum{72.99}\ablsub{+0.03} & \ablnum{8.07}\ablsub{\textminus40.22} & \ablnum{57.74}\ablsub{\textminus10.14} \\
\end{tabular}}\par\nointerlineskip
{\setlength{\tabcolsep}{2.35pt}%
\begin{tabular}{@{}cl*{3}{w{c}{46.00pt}}}
\midrule[\heavyrulewidth]
& Variant & VOC & ADE20K & COCO \\
\midrule
\multirow{3}{*}{W2} & \cellcolor{oursw}SpikeQ (Ours) & \cellcolor{oursw}\ablnum{65.74}\ablsub{} & \cellcolor{oursw}\ablnum{29.32}\ablsub{} & \cellcolor{oursw}\ablnum{21.30}\ablsub{} \\
 & w/o $\mathcal L_{\mathrm{margin}}$ & \ablnum{65.66}\ablsub{\textminus0.08} & \ablnum{28.04}\ablsub{\textminus1.28} & \ablnum{19.29}\ablsub{\textminus2.01} \\
 & w/o $\mathcal L_{\mathrm{pop}}$ & \ablnum{62.55}\ablsub{\textminus3.19} & \ablnum{26.14}\ablsub{\textminus3.18} & \ablnum{20.87}\ablsub{\textminus0.43} \\
\midrule
\multirow{3}{*}{\makecell{Pre\\V2}} & \cellcolor{oursv}SpikeQ (Ours) & \cellcolor{oursv}\ablnum{69.71}\ablsub{} & \cellcolor{oursv}\ablnum{33.44}\ablsub{} & \cellcolor{oursv}\ablnum{31.14}\ablsub{} \\
 & w/o anchoring & \ablnum{18.78}\ablsub{\textminus50.93} & \ablnum{2.96}\ablsub{\textminus30.48} & \ablnum{0.00}\ablsub{\textminus31.14} \\
 & w/o reset-aware $\Delta_v$ & \ablnum{68.98}\ablsub{\textminus0.73} & \ablnum{33.31}\ablsub{\textminus0.13} & \ablnum{31.07}\ablsub{\textminus0.07} \\
\midrule
\multirow{2}{*}{\makecell{Post\\V2}} & \cellcolor{oursv}SpikeQ (Ours) & \cellcolor{oursv}\ablnum{69.79}\ablsub{} & \cellcolor{oursv}\ablnum{33.67}\ablsub{} & \cellcolor{oursv}\ablnum{31.24}\ablsub{} \\
 & MSE instead of $\mathcal L_v$ & \ablnum{65.57}\ablsub{\textminus4.22} & \ablnum{31.73}\ablsub{\textminus1.94} & \ablnum{30.33}\ablsub{\textminus0.91} \\
\bottomrule
\end{tabular}}
\caption{Ablation of SpikeQ at W2 (full-precision membrane potentials) and V2 (full-precision weights). Pre V2 and Post V2 use pre-fire and post-reset quantization. Subscripts give the change from SpikeQ. Top-1 accuracy, mIoU, and AP (\%).}
\label{tab:ablation}
\end{table}

\subsection{Ablation Study}
\label{sec:ablation}
\Cref{tab:ablation} removes or changes one component of SpikeQ at a time, at W2 for weights and at V2 for membrane potentials.

\paragraph{Weight terms.}
The margin and population terms matter on different tasks. Removing the margin term costs up to 5.29 points on classification and detection, and the cost grows with task complexity, like the gain of SpikeQ in \cref{sec:exp-w}, since this term helps correct wrong spikes that a complex task has less redundancy to absorb. The population term matters more on segmentation, where removing it costs about 3.2 mIoU.

\paragraph{Pre-fire grid rules.}
Without threshold anchoring, the grid has a level at $\theta$, so membrane potentials within half a step below $\theta$ are rounded up to $\theta$ and fire. Although $\mathcal L_v$ still selects the step by flipped spikes, SpikeQ collapses on every hard-reset benchmark, so anchoring is what prevents the collapse. Without the reset-aware step, each reset type uses the step rule of the other, which slightly lowers accuracy on most benchmarks.

\paragraph{Post-reset calibration.}
Under post-reset quantization, replacing $\mathcal L_v$ with the MSE among the same candidates lowers the accuracy of SpikeQ on every hard-reset benchmark by up to 40.22 points. The MSE often selects a grid that rounds the membrane potentials just below $\theta$ to zero, so their neurons miss their spikes at the next time step. On the soft-reset CIFAR models, both criteria are equally accurate, as no candidate grid there rounds these membrane potentials to zero.

\section{Conclusion}
This paper shows that quantization errors reduce the accuracy of SNNs through the spikes they flip rather than through their magnitude. SpikeQ therefore selects the weight levels that cause the fewest spike errors. Its membrane grid makes the threshold a rounding boundary, so rounding the membrane potential before firing never moves it across the threshold, and matches its levels to the reset rule, so the post-reset state needs no further rounding. Across image classification, semantic segmentation, and object detection, SpikeQ outperforms existing methods in most settings.

\clearpage
{\small
\bibliographystyle{ieeenat_fullname}
\bibliography{aaai2027}
}

\clearpage
\maketitlesupplementary
\appendix
\raggedbottom
\setcounter{table}{0}
\renewcommand{\thetable}{S\arabic{table}}
\setcounter{figure}{0}
\renewcommand{\thefigure}{S\arabic{figure}}
\renewcommand{\floatpagefraction}{0.8}
\makeatletter\setlength{\@fptop}{0pt}\makeatother
\begin{table*}[b]
\centering
\footnotesize
\setlength{\aboverulesep}{0pt}
\setlength{\belowrulesep}{0pt}
\renewcommand{\arraystretch}{1.12}
{\setlength{\tabcolsep}{15.19pt}%
\begin{tabular}{l*{8}{c}}
\toprule
\multirow{2}{*}{W3} & \multicolumn{2}{c}{CIFAR} & \multicolumn{2}{c}{ImageNet} & \multirow{2}{*}{VOC 2012} & \multirow{2}{*}{ADE20K} & \multicolumn{2}{c}{COCO} \\
\cmidrule(lr){2-3}\cmidrule(lr){4-5}\cmidrule(lr){8-9}
& 10 & 100 & SEW & SDT & & & AP & AP$_{50}$ \\
\midrule
FP & 92.60 & 73.78 & 62.89 & 72.04 & 70.86 & 34.91 & 31.38 & 50.86 \\
\midrule
RTN & 90.37 & 42.91 & 0.94 & 0.40 & 56.40 & 23.75 & 5.72 & 15.16 \\
FlowQ~\cite{flowq} & 85.39 & 30.23 & 0.20 & 1.55 & 9.24 & 2.28 & 0.48 & 1.31 \\
GPTQ~\cite{frantar2023gptq} & 91.81 & 69.28 & 57.20 & 57.96 & \tsecond{69.69} & 32.88 & 26.20 & 48.12 \\
SBC~\cite{shi2026sbc} & 91.94 & 69.34 & 57.20 & 58.31 & \tbest{69.85} & \tsecond{32.96} & \tsecond{26.83} & \tsecond{48.22} \\
GPTAQ~\cite{li2025gptaq} & \tsecond{91.97} & \tbest{71.55} & \tsecond{58.82} & \tsecond{61.75} & 69.24 & 32.53 & 24.89 & 46.01 \\
BRECQ~\cite{li2021brecq} & 91.50 & 69.76 & 57.21 & 44.74 & 65.92 & 26.41 & 22.71 & 42.95 \\
PTQ4SNN~\cite{xie2026ptq4snn} & 91.65 & 69.91 & 57.14 & 44.62 & 66.20 & 26.01 & 22.69 & 43.12 \\
\cellcolor{oursw}SpikeQ (Ours) & \cellcolor{oursw}\tbest{92.14} & \cellcolor{oursw}\tsecond{71.54} & \cellcolor{oursw}\tbest{59.13} & \cellcolor{oursw}\tbest{67.55} & \cellcolor{oursw}69.68 & \cellcolor{oursw}\tbest{33.37} & \cellcolor{oursw}\tbest{27.94} & \cellcolor{oursw}\tbest{48.43} \\
\end{tabular}}\par\nointerlineskip
{\setlength{\tabcolsep}{10.32pt}%
\begin{tabular}{l*{7}{c}}
\midrule[\heavyrulewidth]
\multirow{2}{*}{V3} & \multicolumn{2}{c}{CIFAR} & \multicolumn{2}{c}{ImageNet} & \multirow{2}{*}{VOC 2012} & \multirow{2}{*}{ADE20K} & \multirow{2}{*}{COCO AP} \\
\cmidrule(lr){2-3}\cmidrule(lr){4-5}
& 10 & 100 & SEW & SDT & & & \\
\midrule
FP & 92.60 & 73.78 & 62.89 & 72.04 & 70.86 & 34.91 & 31.38 \\
\midrule
MinMax~\cite{jacob2018quantization} & 10.06/81.03 & 1.22/66.69 & 0.10/0.78 & 0.11/45.38 & 0.81/57.24 & 0.11/14.97 & 0.00/29.63 \\
MSE~\cite{nagel2021whitepaper} & 41.88/92.08 & 48.45/72.50 & 0.10/0.11 & 0.10/51.22 & 3.30/68.22 & 0.02/31.84 & 0.00/30.44 \\
FlowQ-M~\cite{flowq} & 57.51/92.57 & 18.87/73.36 & 0.13/17.20 & 2.26/70.49 & 5.05/69.02 & 2.38/33.98 & 1.43/31.00 \\
SQUAT~\cite{venkatesh2024squat} & 43.98/92.43 & 24.31/72.46 & 0.16/43.24 & 0.12/71.17 & 1.44/69.49 & 0.04/34.08 & 0.07/\tsecond{31.31} \\
PTQ4SNN-M$^\dagger$~\cite{xie2026ptq4snn} & 77.40/\tbest{92.78} & 23.14/73.37 & 1.22/29.33 & 18.11/65.92 & 17.06/70.48 & \tsecond{9.21}/\tbest{34.72} & 14.98/31.12 \\
\quad+MPBA$^\dagger$~\cite{xie2026ptq4snn} & 10.00/91.96 & 2.15/72.65 & 1.67/24.57 & 0.77/69.28 & 3.58/70.27 & 0.36/34.41 & 3.59/31.09 \\
\cellcolor{oursv}SpikeQ (Ours) & \cellcolor{oursv}\tbest{92.59}/92.60 & \cellcolor{oursv}\tbest{72.87}/\tbest{73.89} & \cellcolor{oursv}\tsecond{58.08}/\tsecond{58.08} & \cellcolor{oursv}\tsecond{71.43}/\tsecond{71.40} & \cellcolor{oursv}\tsecond{70.54}/\tsecond{70.61} & \cellcolor{oursv}\tbest{34.45}/34.48 & \cellcolor{oursv}\tsecond{31.35}/31.30 \\
\cellcolor{oursv}SpikeQ$^\dagger$ (Ours) & \cellcolor{oursv}\tsecond{92.48}/\tsecond{92.72} & \cellcolor{oursv}\tsecond{72.85}/\tsecond{73.78} & \cellcolor{oursv}\tbest{62.18}/\tbest{62.11} & \cellcolor{oursv}\tbest{71.53}/\tbest{71.76} & \cellcolor{oursv}\tbest{70.92}/\tbest{70.74} & \cellcolor{oursv}\tbest{34.45}/\tsecond{34.59} & \cellcolor{oursv}\tbest{31.46}/\tbest{31.35} \\
\bottomrule
\end{tabular}}
\caption{Weight quantization (W3, top) with full-precision membrane potentials and membrane quantization (V3, bottom) with full-precision weights. Top-1 accuracy, mIoU, and AP (\%). Notation as in \cref{tab:cls-w,tab:cls-v}.}
\label{tab:supp-3bit}
\end{table*}

\section{Implementation Details}
\label{sec:supp-impl}
\paragraph{Models and calibration.}
All experiments run on NVIDIA A100 GPUs with 80\,GB of memory. The CIFAR models and both SDT models use LIF neurons, and SEW-ResNet-18 uses IF neurons. Weight calibration uses a single seed, and on CIFAR-10 at W2, SpikeQ averages $90.52\pm0.19$ over three calibration seeds. All COCO results, including full precision, are evaluated in FP32 with TF32 disabled.

\paragraph{Weight quantization.}
Weights are quantized per output channel. GPTQ and SBC use their default min-max grids, and GPTAQ uses $\alpha=0.25$, the default of its official implementation. BRECQ, PTQ4SNN, and SpikeQ share the same clipping search for the weight grid, which minimizes the $L_{2.4}$ weight error, so they differ only in the selected levels. BRECQ and PTQ4SNN use 20k reconstruction iterations (5k on COCO). On COCO, all methods keep the first layer and the detection head at 8 bits with round-to-nearest.

\paragraph{SpikeQ weight search.}
After each sweep, SpikeQ updates
\[
\mu\leftarrow\mathrm{clip}\Bigl(\mu+\frac{0.5\,\mathcal L_{\mathrm{pop}}}{\mathcal L_{\mathrm{count}}+\lambda\mathcal L_{\mathrm{margin}}},-0.9,0.9\Bigr),
\]
where dividing by the total error makes the step independent of the layer size. The held-out part is 25\% of the calibration images. SpikeQ uses $\rho=0.5$ at W2, $\rho=0.25$ at W3 and W4, $\lambda=1$, four draws of $\varepsilon$, and at most six sweeps in the second phase.

\paragraph{Simulated layers.}
The simulated input current includes batch normalization, channel-wise pooling, and residual inputs between a layer and its neurons. On large feature maps, SpikeQ samples spatial positions but keeps all time steps. When a channel drives no spiking neurons of its own, as in the non-spiking classifier, SpikeQ keeps the search but replaces \cref{eq:weight-loss} with the output error $\mathcal L_{\mathrm{out}}=\sum_t\|(\widehat W^{\ell}-W^{\ell})\widehat{\mathbf S}^{\ell-1}_t\|^2$.

\paragraph{Membrane quantization.}
Membrane grids are calibrated on 256 images, and for segmentation we sample 16{,}384 neurons per image and layer. MinMax, MSE, FlowQ-M, SQUAT, and SpikeQ are layer-wise. PTQ4SNN-M uses channel-wise scales, and SpikeQ$^\dagger$ lets each channel select its own grid from the same threshold-anchored candidates. +MPBA mixes 2, 4, and 8 bits, and its V2 results use 2.3 bits on average. Under post-reset quantization, our grids keep the same candidates and selection rule, and MinMax and MSE calibrate their ranges on post-reset membrane potentials. For W$b$V$b$, the membrane grids are calibrated on the weight-quantized network (\cref{sec:membrane-method}).

\section{Proof of Reset-Aware Level Allocation}
\label{sec:supp-grid}
Reset-aware level allocation (\cref{sec:membrane-method}) sets $\Delta_v$ so that the post-reset state is a level of $\mathcal G_V$ and needs no further rounding. We prove this below.
\begin{proposition}
Under pre-fire quantization with the grid of \cref{eq:anchored-grid}, $\Delta_v$ from \cref{eq:reset-grid}, and the stated ranges of $m$, the post-reset state is a level of $\mathcal G_V$ under both hard and soft reset, and the $m$ levels $g_{n-m},\ldots,g_{n-1}$ are those between $0$ and $\theta$.
\end{proposition}
\begin{proof}
Without a spike, the post-reset state is $\widehat V_t=Q_v(V_t)$, which is a level. With a spike, $\widehat V_t=g_j$ for some $j\geq n$, since $g_n$ is the lowest level above $\theta$.

Under hard reset, the post-reset state is $0$. With $\Delta_v=\theta/(m+\tfrac12)$, \cref{eq:anchored-grid} gives $g_{n-m-1}=\theta-(m+\tfrac12)\Delta_v=0$, and $n=K-1$ with $m\leq K-2$ ensures $n-m-1\geq0$.

Under soft reset, the post-reset state is $g_j-\theta$. With $\theta=m\Delta_v$, \cref{eq:anchored-grid} gives $g_j-\theta=g_{j-m}$, and $m\leq n\leq j$ ensures $j-m\geq0$. In both cases, the lowest level above $0$ is $g_{n-m}$, which equals $\Delta_v$ under hard reset and $\tfrac12\Delta_v$ under soft reset, so the $m$ levels $g_{n-m},\ldots,g_{n-1}$ are those between $0$ and $\theta$.
\end{proof}

\section{Results at 3 Bits}
\label{sec:supp-3bit}
\Cref{tab:supp-3bit} reports the 3-bit results omitted from the main text, which follow the trends at 2 and 4 bits. SpikeQ gives the most accurate weight quantization in six of the eight W3 columns. Under pre-fire quantization, every baseline stays at least 15 points below full precision at V3, whereas our layer-wise grids stay within 0.91 points of full precision except on SEW-ResNet-18, where SpikeQ$^\dagger$ reaches 62.18\%.

\section{Flipped Spikes and Accuracy}
\label{sec:supp-flip}
\begin{figure}[p]
\centering
\includegraphics{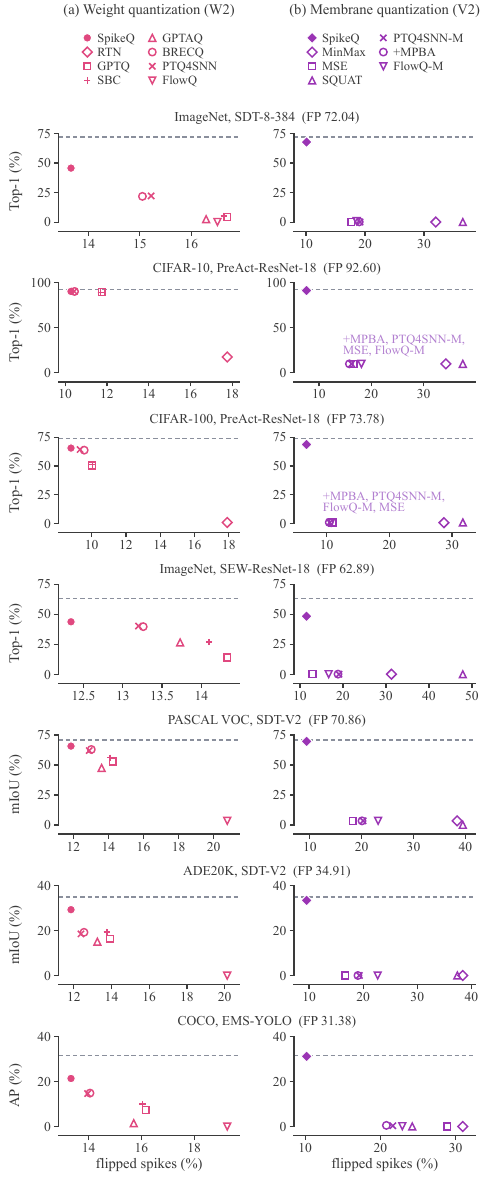}
\caption{Accuracy versus flipped spikes at 2~bits. (a)~Weight quantization (W2) with full-precision membrane potentials. (b)~Pre-fire membrane quantization (V2) with full-precision weights. Flipped spikes are counted against the full-precision network over all neurons and time steps.}
\label{fig:supp-flip-2bit}
\end{figure}
\begin{figure}[p]
\centering
\includegraphics{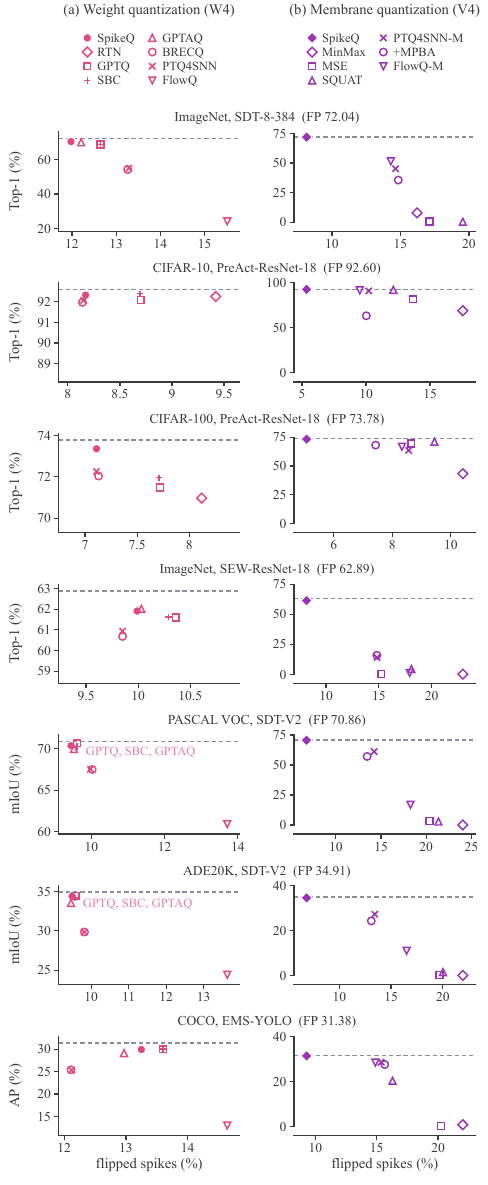}
\caption{Accuracy versus flipped spikes at 4~bits, with panels as in \cref{fig:supp-flip-2bit}. The first row repeats \cref{fig:flip-acc}, and the y-axes of (a) do not start at zero.}
\label{fig:supp-flip-4bit}
\end{figure}
\paragraph{Measurement.}
\Cref{fig:supp-flip-2bit,fig:supp-flip-4bit} extend \cref{fig:flip-acc} to all seven benchmarks at 2 and 4~bits. Flipped spikes are counted against the full-precision network over all neurons and time steps, so they include flips passed on from earlier layers and time steps. Flipped spikes are measured on 2{,}000 evaluation images for CIFAR, 1{,}000 for SEW-ResNet-18, 500 for SDT, and 200 for VOC, ADE20K, and COCO, and accuracy on the full evaluation sets.

\paragraph{Weight quantization.}
At W2, SpikeQ flips the fewest spikes on every benchmark, and accuracy falls as flipped spikes grow (\cref{fig:supp-flip-2bit}, \cref{sec:exp-w}). At W4, the fine step keeps rounding errors small, so most methods flip similar numbers of spikes and reach similar accuracy (\cref{fig:supp-flip-4bit}, \cref{sec:exp-w}). Accuracy drops mainly for the methods that flip more spikes.

\paragraph{Membrane quantization.}
SpikeQ flips the fewest spikes on every benchmark at both V2 and V4, since threshold anchoring removes the band of membrane potentials rounded across $\theta$ (\cref{sec:exp-v}). At V2, every existing grid collapses, so these panels show no graded trend, only that SpikeQ alone keeps most of the full-precision accuracy. At V4, the finer step narrows the band, existing grids partly recover, and accuracy falls as flipped spikes grow on ImageNet, segmentation, and detection. On CIFAR, the accuracy of the existing grids is not ordered by their flipped spikes.

\section{Energy Consumption}
\label{sec:supp-energy}
\begin{table}[t]
\centering
\footnotesize
\setlength{\aboverulesep}{0pt}
\setlength{\belowrulesep}{0pt}
\renewcommand{\arraystretch}{1.12}
\setlength{\tabcolsep}{3.50pt}
\begin{tabular}{llc}
\toprule
Symbol & Operation & Energy (pJ) \\
\midrule
$e_{\mathrm{add}}(32)$ & synaptic addition, 32 bits (floating point) & 0.900 \\
$e_{\mathrm{add}}(4)$ & synaptic addition, 4 bits & 0.0502 \\
$e_{\mathrm{add}}(2)$ & synaptic addition, 2 bits & 0.0477 \\
$e_{\mathrm{wread}}$ & weight read, per bit & 0.0985 \\
$e_{\mathrm{spike}}$ & spike transmission & 0.180 \\
$e_{\mathrm{state}}(b)$ & membrane update, comparison, or reset, $b$ bits & $e_{\mathrm{add}}(b)$ \\
$e_{\mathrm{leak}}$ & leakage, per neuron and time step & 0.002 \\
\bottomrule
\end{tabular}
\caption{Energy of one operation in \cref{eq:energy}.}
\label{tab:energy-constants}
\end{table}

\providecommand{\ebn}[2]{\makebox[#1][r]{#2}}
\ifdefined\ebwA\else\newlength{\ebwA}\fi
\ifdefined\ebwB\else\newlength{\ebwB}\fi
\ifdefined\ebwC\else\newlength{\ebwC}\fi
\ifdefined\ebwD\else\newlength{\ebwD}\fi
\ifdefined\ebwE\else\newlength{\ebwE}\fi
\ifdefined\ebwF\else\newlength{\ebwF}\fi
\ifdefined\ebwG\else\newlength{\ebwG}\fi
\ifdefined\ebwH\else\newlength{\ebwH}\fi
\begin{table*}[tp]
\centering
\footnotesize
\setlength{\aboverulesep}{0pt}
\setlength{\belowrulesep}{0pt}
\renewcommand{\arraystretch}{1.12}
\settowidth{\ebwA}{000{,}000}\settowidth{\ebwB}{000{,}000}\settowidth{\ebwC}{0.00}\settowidth{\ebwD}{000{,}000}\settowidth{\ebwE}{00.00}\settowidth{\ebwF}{00.0}\settowidth{\ebwG}{0.000}\settowidth{\ebwH}{00.0}
\setlength{\tabcolsep}{4.20pt}
\begin{tabular}{@{}ll*{11}{c}@{}}
\toprule
\multirow{2}{*}{Benchmark} & \multirow{2}{*}{Setting} & \multirow{2}{*}{FR (\%)} & \multicolumn{4}{c}{Energy ($\mu$J)} & \multirow{2}{*}{\makecell{Reduction\\[-1pt]($\times$)}} & \multicolumn{2}{c}{Share of EC (\%)} & \multicolumn{3}{c}{Size (MB)} \\
\cmidrule(lr){4-7}\cmidrule(lr){9-10}\cmidrule(lr){11-13}
&&& Compute & Memory & Neuron & EC & & Memory & Non-spiking & Weights & Membrane & Footprint \\
\midrule
\multirow{7}{*}{\makecell[l]{CIFAR-10\\[-1pt]{\scriptsize PreAct-ResNet-18}}} & FP & 20.55 & \ebn{\ebwA}{331} & \ebn{\ebwB}{1{,}227} & \ebn{\ebwC}{2.42} & \ebn{\ebwD}{1{,}560} & \ebn{\ebwE}{} & 78.6 & 1.9 & \ebn{\ebwF}{42.6} & \ebn{\ebwG}{2.13} & \ebn{\ebwH}{44.8} \\
 & W4 & 20.50 & \ebn{\ebwA}{18.4} & \ebn{\ebwB}{211} & \ebn{\ebwC}{2.42} & \ebn{\ebwD}{231} & \ebn{\ebwE}{\textcolor[gray]{0.55}{6.75}} & 91.0 & 1.9 & \ebn{\ebwF}{5.39} & \ebn{\ebwG}{2.13} & \ebn{\ebwH}{7.52} \\
 & V4 & 20.51 & \ebn{\ebwA}{330} & \ebn{\ebwB}{1{,}223} & \ebn{\ebwC}{0.14} & \ebn{\ebwD}{1{,}554} & \ebn{\ebwE}{\textcolor[gray]{0.55}{1.00}} & 78.7 & 1.9 & \ebn{\ebwF}{42.6} & \ebn{\ebwG}{0.266} & \ebn{\ebwH}{42.9} \\
 & W4V4 & 20.46 & \ebn{\ebwA}{18.4} & \ebn{\ebwB}{210} & \ebn{\ebwC}{0.14} & \ebn{\ebwD}{228} & \ebn{\ebwE}{\textcolor[gray]{0.55}{6.83}} & 91.9 & 1.9 & \ebn{\ebwF}{5.39} & \ebn{\ebwG}{0.266} & \ebn{\ebwH}{5.66} \\
 & W2 & 20.27 & \ebn{\ebwA}{17.3} & \ebn{\ebwB}{136} & \ebn{\ebwC}{2.42} & \ebn{\ebwD}{156} & \ebn{\ebwE}{\textcolor[gray]{0.55}{10.00}} & 87.4 & 1.9 & \ebn{\ebwF}{2.73} & \ebn{\ebwG}{2.13} & \ebn{\ebwH}{4.86} \\
 & V2 & 19.99 & \ebn{\ebwA}{320} & \ebn{\ebwB}{1{,}183} & \ebn{\ebwC}{0.13} & \ebn{\ebwD}{1{,}502} & \ebn{\ebwE}{\textcolor[gray]{0.55}{1.04}} & 78.7 & 2.0 & \ebn{\ebwF}{42.6} & \ebn{\ebwG}{0.133} & \ebn{\ebwH}{42.8} \\
 & W2V2 & 19.50 & \ebn{\ebwA}{16.5} & \ebn{\ebwB}{130} & \ebn{\ebwC}{0.13} & \ebn{\ebwD}{147} & \ebn{\ebwE}{\textcolor[gray]{0.55}{10.61}} & 88.6 & 2.0 & \ebn{\ebwF}{2.73} & \ebn{\ebwG}{0.133} & \ebn{\ebwH}{2.86} \\
\midrule
\multirow{7}{*}{\makecell[l]{CIFAR-100\\[-1pt]{\scriptsize PreAct-ResNet-18}}} & FP & 9.48 & \ebn{\ebwA}{189} & \ebn{\ebwB}{698} & \ebn{\ebwC}{2.20} & \ebn{\ebwD}{889} & \ebn{\ebwE}{} & 78.5 & 3.4 & \ebn{\ebwF}{42.8} & \ebn{\ebwG}{2.13} & \ebn{\ebwH}{44.9} \\
 & W4 & 9.43 & \ebn{\ebwA}{10.4} & \ebn{\ebwB}{119} & \ebn{\ebwC}{2.20} & \ebn{\ebwD}{132} & \ebn{\ebwE}{\textcolor[gray]{0.55}{6.74}} & 90.4 & 3.3 & \ebn{\ebwF}{5.41} & \ebn{\ebwG}{2.13} & \ebn{\ebwH}{7.54} \\
 & V4 & 9.46 & \ebn{\ebwA}{188} & \ebn{\ebwB}{695} & \ebn{\ebwC}{0.13} & \ebn{\ebwD}{883} & \ebn{\ebwE}{\textcolor[gray]{0.55}{1.01}} & 78.7 & 3.4 & \ebn{\ebwF}{42.8} & \ebn{\ebwG}{0.266} & \ebn{\ebwH}{43.1} \\
 & W4V4 & 9.41 & \ebn{\ebwA}{10.4} & \ebn{\ebwB}{119} & \ebn{\ebwC}{0.13} & \ebn{\ebwD}{129} & \ebn{\ebwE}{\textcolor[gray]{0.55}{6.88}} & 91.9 & 3.4 & \ebn{\ebwF}{5.41} & \ebn{\ebwG}{0.266} & \ebn{\ebwH}{5.68} \\
 & W2 & 9.28 & \ebn{\ebwA}{9.6} & \ebn{\ebwB}{76.2} & \ebn{\ebwC}{2.20} & \ebn{\ebwD}{88.0} & \ebn{\ebwE}{\textcolor[gray]{0.55}{10.10}} & 86.6 & 3.4 & \ebn{\ebwF}{2.74} & \ebn{\ebwG}{2.13} & \ebn{\ebwH}{4.87} \\
 & V2 & 9.10 & \ebn{\ebwA}{180} & \ebn{\ebwB}{666} & \ebn{\ebwC}{0.12} & \ebn{\ebwD}{846} & \ebn{\ebwE}{\textcolor[gray]{0.55}{1.05}} & 78.7 & 3.6 & \ebn{\ebwF}{42.8} & \ebn{\ebwG}{0.133} & \ebn{\ebwH}{42.9} \\
 & W2V2 & 8.90 & \ebn{\ebwA}{9.3} & \ebn{\ebwB}{73.2} & \ebn{\ebwC}{0.12} & \ebn{\ebwD}{82.6} & \ebn{\ebwE}{\textcolor[gray]{0.55}{10.76}} & 88.6 & 3.6 & \ebn{\ebwF}{2.74} & \ebn{\ebwG}{0.133} & \ebn{\ebwH}{2.87} \\
\midrule
\multirow{7}{*}{\makecell[l]{ImageNet\\[-1pt]{\scriptsize SEW-ResNet-18}}} & FP & 12.95 & \ebn{\ebwA}{1{,}264} & \ebn{\ebwB}{4{,}681} & \ebn{\ebwC}{10.1} & \ebn{\ebwD}{5{,}955} & \ebn{\ebwE}{} & 78.6 & 8.4 & \ebn{\ebwF}{44.6} & \ebn{\ebwG}{9.48} & \ebn{\ebwH}{54.1} \\
 & W4 & 12.91 & \ebn{\ebwA}{70.3} & \ebn{\ebwB}{804} & \ebn{\ebwC}{10.1} & \ebn{\ebwD}{884} & \ebn{\ebwE}{\textcolor[gray]{0.55}{6.74}} & 90.9 & 8.4 & \ebn{\ebwF}{5.63} & \ebn{\ebwG}{9.48} & \ebn{\ebwH}{15.1} \\
 & V4 & 12.96 & \ebn{\ebwA}{1{,}265} & \ebn{\ebwB}{4{,}681} & \ebn{\ebwC}{0.58} & \ebn{\ebwD}{5{,}946} & \ebn{\ebwE}{\textcolor[gray]{0.55}{1.00}} & 78.7 & 8.4 & \ebn{\ebwF}{44.6} & \ebn{\ebwG}{1.18} & \ebn{\ebwH}{45.8} \\
 & W4V4 & 12.92 & \ebn{\ebwA}{70.3} & \ebn{\ebwB}{804} & \ebn{\ebwC}{0.58} & \ebn{\ebwD}{875} & \ebn{\ebwE}{\textcolor[gray]{0.55}{6.81}} & 91.9 & 8.5 & \ebn{\ebwF}{5.63} & \ebn{\ebwG}{1.18} & \ebn{\ebwH}{6.81} \\
 & W2 & 12.53 & \ebn{\ebwA}{64.7} & \ebn{\ebwB}{511} & \ebn{\ebwC}{10.1} & \ebn{\ebwD}{586} & \ebn{\ebwE}{\textcolor[gray]{0.55}{10.17}} & 87.2 & 8.6 & \ebn{\ebwF}{2.85} & \ebn{\ebwG}{9.48} & \ebn{\ebwH}{12.3} \\
 & V2 & 13.16 & \ebn{\ebwA}{1{,}248} & \ebn{\ebwB}{4{,}619} & \ebn{\ebwC}{0.56} & \ebn{\ebwD}{5{,}867} & \ebn{\ebwE}{\textcolor[gray]{0.55}{1.01}} & 78.7 & 8.5 & \ebn{\ebwF}{44.6} & \ebn{\ebwG}{0.592} & \ebn{\ebwH}{45.2} \\
 & W2V2 & 12.84 & \ebn{\ebwA}{64.6} & \ebn{\ebwB}{510} & \ebn{\ebwC}{0.55} & \ebn{\ebwD}{576} & \ebn{\ebwE}{\textcolor[gray]{0.55}{10.35}} & 88.7 & 8.7 & \ebn{\ebwF}{2.85} & \ebn{\ebwG}{0.592} & \ebn{\ebwH}{3.44} \\
\midrule
\multirow{7}{*}{\makecell[l]{ImageNet\\[-1pt]{\scriptsize SDT-8-384}}} & FP & 10.93 & \ebn{\ebwA}{3{,}908} & \ebn{\ebwB}{14{,}467} & \ebn{\ebwC}{38.9} & \ebn{\ebwD}{18{,}413} & \ebn{\ebwE}{} & 78.6 & 6.0 & \ebn{\ebwF}{64.2} & \ebn{\ebwG}{37.1} & \ebn{\ebwH}{101} \\
 & W4 & 10.91 & \ebn{\ebwA}{217} & \ebn{\ebwB}{2{,}485} & \ebn{\ebwC}{38.9} & \ebn{\ebwD}{2{,}741} & \ebn{\ebwE}{\textcolor[gray]{0.55}{6.72}} & 90.7 & 5.9 & \ebn{\ebwF}{8.39} & \ebn{\ebwG}{37.1} & \ebn{\ebwH}{45.4} \\
 & V4 & 10.97 & \ebn{\ebwA}{3{,}916} & \ebn{\ebwB}{14{,}497} & \ebn{\ebwC}{2.24} & \ebn{\ebwD}{18{,}415} & \ebn{\ebwE}{\textcolor[gray]{0.55}{1.00}} & 78.7 & 6.0 & \ebn{\ebwF}{64.2} & \ebn{\ebwG}{4.63} & \ebn{\ebwH}{68.8} \\
 & W4V4 & 10.95 & \ebn{\ebwA}{218} & \ebn{\ebwB}{2{,}489} & \ebn{\ebwC}{2.24} & \ebn{\ebwD}{2{,}709} & \ebn{\ebwE}{\textcolor[gray]{0.55}{6.80}} & 91.9 & 6.0 & \ebn{\ebwF}{8.39} & \ebn{\ebwG}{4.63} & \ebn{\ebwH}{13.0} \\
 & W2 & 10.69 & \ebn{\ebwA}{203} & \ebn{\ebwB}{1{,}606} & \ebn{\ebwC}{38.8} & \ebn{\ebwD}{1{,}848} & \ebn{\ebwE}{\textcolor[gray]{0.55}{9.97}} & 86.9 & 6.0 & \ebn{\ebwF}{4.40} & \ebn{\ebwG}{37.1} & \ebn{\ebwH}{41.4} \\
 & V2 & 10.61 & \ebn{\ebwA}{3{,}795} & \ebn{\ebwB}{14{,}050} & \ebn{\ebwC}{2.13} & \ebn{\ebwD}{17{,}848} & \ebn{\ebwE}{\textcolor[gray]{0.55}{1.03}} & 78.7 & 6.2 & \ebn{\ebwF}{64.2} & \ebn{\ebwG}{2.32} & \ebn{\ebwH}{66.5} \\
 & W2V2 & 10.34 & \ebn{\ebwA}{197} & \ebn{\ebwB}{1{,}559} & \ebn{\ebwC}{2.12} & \ebn{\ebwD}{1{,}759} & \ebn{\ebwE}{\textcolor[gray]{0.55}{10.47}} & 88.7 & 6.3 & \ebn{\ebwF}{4.40} & \ebn{\ebwG}{2.32} & \ebn{\ebwH}{6.72} \\
\midrule
\multirow{7}{*}{\makecell[l]{VOC 2012\\[-1pt]{\scriptsize SDT-V2}}} & FP & 16.60 & \ebn{\ebwA}{89{,}357} & \ebn{\ebwB}{330{,}819} & \ebn{\ebwC}{400} & \ebn{\ebwD}{420{,}575} & \ebn{\ebwE}{} & 78.7 & 35.9 & \ebn{\ebwF}{224} & \ebn{\ebwG}{362} & \ebn{\ebwH}{586} \\
 & W4 & 16.54 & \ebn{\ebwA}{4{,}976} & \ebn{\ebwB}{56{,}899} & \ebn{\ebwC}{399} & \ebn{\ebwD}{62{,}274} & \ebn{\ebwE}{\textcolor[gray]{0.55}{6.75}} & 91.4 & 35.8 & \ebn{\ebwF}{28.9} & \ebn{\ebwG}{362} & \ebn{\ebwH}{391} \\
 & V4 & 16.69 & \ebn{\ebwA}{89{,}660} & \ebn{\ebwB}{331{,}941} & \ebn{\ebwC}{23.0} & \ebn{\ebwD}{421{,}624} & \ebn{\ebwE}{\textcolor[gray]{0.55}{1.00}} & 78.7 & 35.9 & \ebn{\ebwF}{224} & \ebn{\ebwG}{45.3} & \ebn{\ebwH}{269} \\
 & W4V4 & 16.60 & \ebn{\ebwA}{4{,}990} & \ebn{\ebwB}{57{,}057} & \ebn{\ebwC}{23.0} & \ebn{\ebwD}{62{,}070} & \ebn{\ebwE}{\textcolor[gray]{0.55}{6.78}} & 91.9 & 35.9 & \ebn{\ebwF}{28.9} & \ebn{\ebwG}{45.3} & \ebn{\ebwH}{74.2} \\
 & W2 & 16.10 & \ebn{\ebwA}{4{,}688} & \ebn{\ebwB}{37{,}051} & \ebn{\ebwC}{398} & \ebn{\ebwD}{42{,}137} & \ebn{\ebwE}{\textcolor[gray]{0.55}{9.98}} & 87.9 & 36.0 & \ebn{\ebwF}{15.0} & \ebn{\ebwG}{362} & \ebn{\ebwH}{377} \\
 & V2 & 16.15 & \ebn{\ebwA}{87{,}817} & \ebn{\ebwB}{325{,}118} & \ebn{\ebwC}{21.8} & \ebn{\ebwD}{412{,}957} & \ebn{\ebwE}{\textcolor[gray]{0.55}{1.02}} & 78.7 & 36.6 & \ebn{\ebwF}{224} & \ebn{\ebwG}{22.6} & \ebn{\ebwH}{246} \\
 & W2V2 & 15.64 & \ebn{\ebwA}{4{,}607} & \ebn{\ebwB}{36{,}412} & \ebn{\ebwC}{21.7} & \ebn{\ebwD}{41{,}041} & \ebn{\ebwE}{\textcolor[gray]{0.55}{10.25}} & 88.7 & 37.0 & \ebn{\ebwF}{15.0} & \ebn{\ebwG}{22.6} & \ebn{\ebwH}{37.6} \\
\midrule
\multirow{7}{*}{\makecell[l]{ADE20K\\[-1pt]{\scriptsize SDT-V2}}} & FP & 15.31 & \ebn{\ebwA}{118{,}221} & \ebn{\ebwB}{437{,}680} & \ebn{\ebwC}{544} & \ebn{\ebwD}{556{,}445} & \ebn{\ebwE}{} & 78.7 & 37.0 & \ebn{\ebwF}{224} & \ebn{\ebwG}{499} & \ebn{\ebwH}{723} \\
 & W4 & 15.26 & \ebn{\ebwA}{6{,}582} & \ebn{\ebwB}{75{,}260} & \ebn{\ebwC}{544} & \ebn{\ebwD}{82{,}386} & \ebn{\ebwE}{\textcolor[gray]{0.55}{6.75}} & 91.4 & 36.9 & \ebn{\ebwF}{28.9} & \ebn{\ebwG}{499} & \ebn{\ebwH}{528} \\
 & V4 & 15.26 & \ebn{\ebwA}{117{,}813} & \ebn{\ebwB}{436{,}170} & \ebn{\ebwC}{31.3} & \ebn{\ebwD}{554{,}014} & \ebn{\ebwE}{\textcolor[gray]{0.55}{1.00}} & 78.7 & 37.2 & \ebn{\ebwF}{224} & \ebn{\ebwG}{62.4} & \ebn{\ebwH}{286} \\
 & W4V4 & 15.21 & \ebn{\ebwA}{6{,}558} & \ebn{\ebwB}{74{,}990} & \ebn{\ebwC}{31.3} & \ebn{\ebwD}{81{,}580} & \ebn{\ebwE}{\textcolor[gray]{0.55}{6.82}} & 91.9 & 37.2 & \ebn{\ebwF}{28.9} & \ebn{\ebwG}{62.4} & \ebn{\ebwH}{91.4} \\
 & W2 & 15.23 & \ebn{\ebwA}{6{,}281} & \ebn{\ebwB}{49{,}638} & \ebn{\ebwC}{544} & \ebn{\ebwD}{56{,}463} & \ebn{\ebwE}{\textcolor[gray]{0.55}{9.86}} & 87.9 & 36.6 & \ebn{\ebwF}{15.0} & \ebn{\ebwG}{499} & \ebn{\ebwH}{514} \\
 & V2 & 15.52 & \ebn{\ebwA}{119{,}458} & \ebn{\ebwB}{442{,}261} & \ebn{\ebwC}{29.9} & \ebn{\ebwD}{561{,}749} & \ebn{\ebwE}{\textcolor[gray]{0.55}{0.99}} & 78.7 & 36.7 & \ebn{\ebwF}{224} & \ebn{\ebwG}{31.2} & \ebn{\ebwH}{255} \\
 & W2V2 & 15.37 & \ebn{\ebwA}{6{,}303} & \ebn{\ebwB}{49{,}816} & \ebn{\ebwC}{29.9} & \ebn{\ebwD}{56{,}149} & \ebn{\ebwE}{\textcolor[gray]{0.55}{9.91}} & 88.7 & 36.8 & \ebn{\ebwF}{15.0} & \ebn{\ebwG}{31.2} & \ebn{\ebwH}{46.2} \\
\midrule
\multirow{7}{*}{\makecell[l]{COCO\\[-1pt]{\scriptsize EMS-YOLO}}} & FP & 25.64 & \ebn{\ebwA}{32{,}613} & \ebn{\ebwB}{120{,}740} & \ebn{\ebwC}{165} & \ebn{\ebwD}{153{,}518} & \ebn{\ebwE}{} & 78.6 & 12.4 & \ebn{\ebwF}{129} & \ebn{\ebwG}{139} & \ebn{\ebwH}{269} \\
 & W4 & 25.58 & \ebn{\ebwA}{1{,}812} & \ebn{\ebwB}{22{,}486} & \ebn{\ebwC}{165} & \ebn{\ebwD}{24{,}464} & \ebn{\ebwE}{\textcolor[gray]{0.55}{6.28}} & 91.9 & 18.6 & \ebn{\ebwF}{16.4} & \ebn{\ebwG}{139} & \ebn{\ebwH}{156} \\
 & V4 & 25.46 & \ebn{\ebwA}{32{,}275} & \ebn{\ebwB}{119{,}491} & \ebn{\ebwC}{9.49} & \ebn{\ebwD}{151{,}775} & \ebn{\ebwE}{\textcolor[gray]{0.55}{1.01}} & 78.7 & 12.5 & \ebn{\ebwF}{129} & \ebn{\ebwG}{17.4} & \ebn{\ebwH}{147} \\
 & W4V4 & 25.40 & \ebn{\ebwA}{1{,}793} & \ebn{\ebwB}{22{,}271} & \ebn{\ebwC}{9.48} & \ebn{\ebwD}{24{,}074} & \ebn{\ebwE}{\textcolor[gray]{0.55}{6.38}} & 92.5 & 18.9 & \ebn{\ebwF}{16.4} & \ebn{\ebwG}{17.4} & \ebn{\ebwH}{33.8} \\
 & W2 & 25.35 & \ebn{\ebwA}{1{,}714} & \ebn{\ebwB}{16{,}109} & \ebn{\ebwC}{165} & \ebn{\ebwD}{17{,}989} & \ebn{\ebwE}{\textcolor[gray]{0.55}{8.53}} & 89.6 & 25.4 & \ebn{\ebwF}{8.38} & \ebn{\ebwG}{139} & \ebn{\ebwH}{148} \\
 & V2 & 25.38 & \ebn{\ebwA}{32{,}224} & \ebn{\ebwB}{119{,}300} & \ebn{\ebwC}{9.03} & \ebn{\ebwD}{151{,}533} & \ebn{\ebwE}{\textcolor[gray]{0.55}{1.01}} & 78.7 & 12.5 & \ebn{\ebwF}{129} & \ebn{\ebwG}{8.70} & \ebn{\ebwH}{138} \\
 & W2V2 & 25.11 & \ebn{\ebwA}{1{,}695} & \ebn{\ebwB}{15{,}956} & \ebn{\ebwC}{9.01} & \ebn{\ebwD}{17{,}660} & \ebn{\ebwE}{\textcolor[gray]{0.55}{8.69}} & 90.4 & 25.8 & \ebn{\ebwF}{8.38} & \ebn{\ebwG}{8.70} & \ebn{\ebwH}{17.1} \\
\bottomrule
\end{tabular}
\caption{Energy breakdown, model size, and memory of membrane potentials per image. W$b$V$b$ are the settings of \cref{tab:wv}, and W$b$ and V$b$ quantize only the weights or only the membrane potentials. Compute, memory, and neuron are the three terms of \cref{eq:energy}. FR is the mean firing rate over all spiking neurons. Reduction is the EC of FP divided by the EC of the setting. Non-spiking is the share of layers with non-spiking inputs, which are costed with \mbox{$\mathrm{FR}_\ell=1$} and overlap the compute and memory terms. Weights is the model size, Membrane is the memory that stores the membrane potentials, one per neuron with all neurons held at once, and Footprint is their sum, as in \cref{tab:wv}.}
\label{tab:energy-breakdown}
\end{table*}

\paragraph{Energy model.}
Following~\cite{yan2024reconsidering,yan2026otters}, a weight layer $\ell$ performs $\mathrm{SOP}_\ell=O_\ell\,\mathrm{FR}_\ell\,T$ synaptic operations, where $O_\ell$ is its number of multiply-accumulate operations in one time step and $\mathrm{FR}_\ell$ is the firing rate of its input. A spiking layer $\ell$ has $P_\ell$ neurons with firing rate $\mathrm{FR}^{\mathrm{out}}_\ell$. With weight bit width $b_w$ and membrane bit width $b_v$, the three terms of \cref{eq:ec} are
\begin{equation}
\begin{aligned}
E_{\mathrm{compute}}&=\sum_\ell\mathrm{SOP}_\ell\,e_{\mathrm{add}}(b_w),\\
E_{\mathrm{memory}}&=\sum_\ell\mathrm{SOP}_\ell\bigl[b_w e_{\mathrm{wread}}+e_{\mathrm{spike}}\bigr],\\
E_{\mathrm{neuron}}&=\sum_\ell P_\ell T\bigl[e_{\mathrm{state}}(b_v)(1+\mathrm{FR}^{\mathrm{out}}_\ell)+e_{\mathrm{leak}}\bigr].
\end{aligned}
\label{eq:energy}
\end{equation}
Each synaptic operation adds one weight to a membrane potential, reads this $b_w$-bit weight, and transmits one spike. Each neuron updates and compares its membrane potential at every time step, resets it when it fires, and leaks. \Cref{tab:energy-constants} lists the energy of each operation. A 32-bit synaptic addition is costed as a floating-point addition~\cite{horowitz2014computing}, because a binary input spike turns the multiplication into a conditional addition. Under this model, two quantization methods with the same bit widths differ in EC only through their firing rates.

\paragraph{Measurement.}
$O_\ell$ and $P_\ell$ are counted with the evaluation preprocessing, at $224\times224$ for ImageNet, on whole images for segmentation, and on images letterboxed to 640 pixels for COCO, and are averaged over 16 evaluation images. Firing rates are measured on the full evaluation sets of CIFAR and VOC, on 5{,}000 ImageNet evaluation images, and on 500 evaluation images of ADE20K and COCO. Layers whose inputs are not spikes are costed with $\mathrm{FR}_\ell=1$. They are the first layer of each network, the classifiers after pooling, the detection head of EMS-YOLO, and the SDT-V2 convolutions that follow another convolution without spiking neurons in between. The bit widths are those of the quantized weights and of the layer-wise pre-fire membrane grids, and W$b$ and V$b$ keep the other quantity at 32 bits. The 8-bit first layer and detection head of EMS-YOLO are costed with $e_{\mathrm{add}}(4)$, since \cref{tab:energy-constants} has no 8-bit value. Products of two activations in attention involve no weights and are not counted.

\paragraph{Energy savings.}
At full precision, memory takes 79\% of EC (\cref{tab:energy-breakdown}), mostly for weight reads, whose cost is proportional to $b_w$. Weight quantization therefore provides almost the entire saving, as W4 and W2 alone lower EC by \mbox{6.3--6.8$\times$} and \mbox{8.5--10.2$\times$}, close to W4V4 and W2V2. Membrane quantization alone changes EC by less than 5\%, since its bit width enters only the neuron term, at most 0.25\% of EC at full precision. Going from W4V4 to W2V2 lowers EC by only \mbox{1.4--1.6$\times$}, because a spike transmission costs the same at every bit width and takes 42\% of the energy of a 2-bit synaptic operation.

\paragraph{Memory.}
The model size counts each weight at its bit width, one 32-bit scale per output channel, and the other parameters at 32 bits, with batch normalization folded and 1\,MB being $2^{20}$ bytes. Neuromorphic chips also keep the membrane potential of every neuron in on-chip SRAM across time steps~\cite{merolla2014million,davies2018loihi}, which \cref{eq:energy} does not count. The Membrane column of \cref{tab:energy-breakdown} gives this memory for one image, and the memory footprint of \cref{tab:wv} adds it to the model size. On segmentation and detection, the large feature maps make this memory exceed the model at full precision, so after W2 the membrane potentials take most of the footprint, and W2V2 reduces it from 377 to 37.6\,MB on VOC.

\end{document}